%% file: main.tex
\documentclass[10pt]{article} 

\usepackage[preprint]{tmlr}

\input{math_commands.tex}

\usepackage{hyperref}
\usepackage{url}

\usepackage{amsthm}
\usepackage{booktabs}
\usepackage{enumitem}
\usepackage{color}
\usepackage{graphicx}

\usepackage{float}
\usepackage{mathtools}
\usepackage{thmtools}
\usepackage{thm-restate}
\usepackage{tikz}

\usepackage[linesnumbered,ruled,norelsize]{algorithm2e}
\usepackage[group-separator={,}]{siunitx}
\usepackage[mode=buildnew]{standalone}

\definecolor{grey}{HTML}{BBBBBB}
\definecolor{red}{HTML}{EE6677}
\definecolor{green}{HTML}{228833}
\definecolor{blue}{HTML}{4477AA}

\title{Stochastic Bandits for Egalitarian Assignment}

\author{\name Eugene Lim \email elimwj@comp.nus.edu.sg \\
      \addr Department of Computer Science, \\ National University of Singapore
      \AND
      \name Vincent Y.~F. Tan \email vtan@nus.edu.sg \\
      \addr Department of Mathematics, \\ Department of Electrical and Computer Engineering, \\ National University of Singapore
      \AND
      \name Harold Soh \email harold@comp.nus.edu.sg\\
      \addr Department of Computer Science, \\ National University of Singapore}

\def\month{10}
\def\year{2024}
\def\openreview{\url{https://openreview.net/forum?id=kmKVJl2JWo}}

\begin{document}

\maketitle

\begin{abstract}
We study \texttt{EgalMAB}, an egalitarian assignment problem in the context of stochastic multi-armed bandits. In \texttt{EgalMAB}, an agent is tasked with assigning a set of users to arms. At each time step, the agent must assign exactly one arm to each user such that no two users are assigned to the same arm. Subsequently, each user obtains a reward drawn from the unknown reward distribution associated with its assigned arm. The agent's objective is to maximize the minimum expected cumulative reward among all users over a fixed horizon. This problem has applications in areas such as fairness in job and resource allocations, among others. We design and analyze a UCB-based policy \texttt{EgalUCB} and establish upper bounds on the cumulative regret. In complement, we establish an almost-matching policy-independent impossibility result.
\end{abstract}

\section{Introduction}

The multi-armed bandit (MAB) problem serves as a model for online decision-making under uncertainty, finding applications in diverse domains~\citep{shen2015,durand2018,ding2019,mueller2019,forouzandeh2021}. In the classical stochastic MAB problem, an agent is provided with a set of $K$ arms, each associated with an unknown distribution. At each round, the agent plays an arm and receive a reward drawn from its distribution. The agent's goal is to maximize the expected cumulative reward obtained over a fixed number of time steps $T$.

In our work, we study the problem of {\em egalitarian assignment} in the context of stochastic MABs, which we refer to as \EgalMAB{}. In this scenario, the agent is provided with a set of $U<K$ users. At each time step, the agent must assign exactly one arm to each user such that no two users are assigned to the same arm. Subsequently, each user obtain a reward drawn from the reward distribution associated with its assigned arm. The agent's objective is to \emph{maximize the minimum} expected cumulative reward among all users.

\EgalMAB{} finds applications in various domains, including ensuring fairness in job and resource allocations. Consider the job assignment problem depicted in Figure~\ref{fig:egalmab-intro}. In this scenario, there are $U$ users with recurring jobs of equal load (e.g., hourly database updates) and $K$ shared cloud computing resources with fluctuating computational power (e.g., due to unrelated loads that are running on neighbouring cores of the same physical node~\citep{kousiouris2011}). The agent's objective is to distribute these jobs \emph{fairly} among the available resources such that over $T$ trials, the user who has to wait the longest across all their jobs (i.e., receives the least cumulative reward) is not significantly worse off compared to other users.

Similarly, ride-hailing services present another scenario where fair assignment is desirable. In this scenario, there are $U$ passengers and $K$ drivers who offer varying degrees of experience to passengers (e.g., due to vehicle condition and driver behavior). The agent's objective is to allocate users \emph{fairly} to vehicles such that over $T$ trips, no passenger faces a significantly worse overall experience than others.

\begin{figure}[t]
    \centering
    \includegraphics[width=0.6\columnwidth]{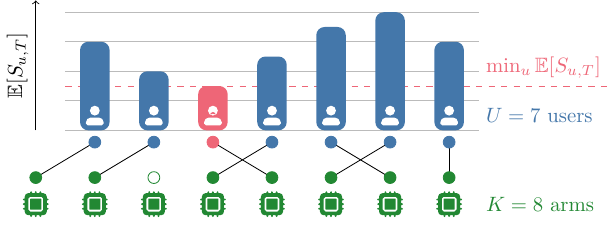}
    \caption{An illustration with \textcolor{green}{$K=8$ arms} and \textcolor{blue}{$U=7$ users}. After time step $T$, each user $u\in[U]$ has some expected cumulative reward $\E[S_{u,T}]$. The agent's objective is to maximize \textcolor{red}{$\min_{u\in[U]} \E[S_{u,T}]$}, which is the minimum expected cumulative reward across all users.}
    \label{fig:egalmab-intro}
\end{figure}

As illustrated in both examples, \EgalMAB{} embodies the principle of egalitarianism, which is also known as the Rawlsian maximin principle~\citep{rawls1971}, in its notion of fairness. Egalitarianism is a fundamental notion of justice based on the \emph{difference principle}, which seeks to \emph{maximize} the welfare of those in society who are the \emph{worst-off}. Likewise, our agent's objective is to maximize the cumulative reward of the user receiving the lowest cumulative reward among all $U$ users. Given our examples above, solving the \EgalMAB{} problem yields a policy that minimizes the overall waiting time of the user who has to wait the longest and maximizes the experience of the passenger who has the most negative encounters. This is in contrast to the classic MAB setting where only \emph{overall utility} is considered. This prompts the fundamental question that our work seeks to answer:
\begin{quote}
    \emph{How can we design an agent's assignment policy that optimizes overall utility for individual users while also ensuring that no user within a group consistently encounters substantially sub-optimal outcomes?}
\end{quote}

Our contributions are summarized as follows. We formally define the \EgalMAB{} problem in Section~\ref{sec:problem} and propose \EgalUCB{}, a UCB-based solution to \EgalMAB{}, in Section~\ref{sec:policy}. In Section~\ref{sec:mainresults}~and~\ref{sec:analysis}, we establish that \EgalUCB{} achieves an expected regret of at most
\begin{equation*}
    O\Biggl(\sqrt{\frac{T\ln(T)\cdot(K-U)}{U}}\Biggr).
\end{equation*}
We also provide a policy-independent lower bound that matches the upper bound up to a multiplicative factor of $\sqrt{1/U}$ and a term logarithmic in $T$. Proof sketches for these results are included in Section~\ref{sec:analysis}. Lastly, empirical validations for these results using both synthetic and real-world data are presented in Section~\ref{sec:experiments}.


\section{Related Works}

\paragraph{MAB with Multiple Plays.} The multiple-play multi-armed bandit (MP-MAB) problem~\citep{anantharam1987,gai2012,chen2016,komiyama2015} expands upon the classical MAB framework by allowing the agent to play $U\geq2$ distinct arms. The MP-MAB problem has been extended in various ways, including combinatorial bandits~\citep{cesa2012,kveton2015a,chen2016}, cascading bandits~\citep{kveton2015b,zheng2017}, and MP-MAB with shareable arms~\citep{wang2022}. An adjacent problem to ours is identifying the top $U$ arms while minimizing regret, which can be framed as a matroid bandit problem \citep{kveton2014} with a uniform matroid of rank $U$. Although similar to \EgalMAB{} in that the agent has to select multiple arms, unlike \EgalMAB{}, these formulations of the MP-MAB problem do not involve \emph{multiple users}; that is, the reward is with respect to the agent instead of users.

\paragraph{MAB with Multiple Users.} The extension of multiple users into the classic MAB problem has been extensively explored in the context of cognitive radio networks~\citep{wassim2009,liu2010,avner2014}. Unlike \EgalMAB{} where a \emph{centralized} agent assigns the arms to the users, these works have focused on scenarios where multiple \emph{decentralized} users interact with a single MAB instance. As a result of this decentralization, it is possible for multiple users to play the same arm simultaneously, resulting in a collision that can negatively affect the received reward.

\paragraph{Fairness in MAB.} The introduction of fairness considerations into the classical MAB problem has garnered significant interest. Much attention has been directed towards addressing fairness concerns that typically involve ensuring that each arm is played a minimum number of times, known as fairness in exposure~\citep{claure2020,chen2020,li2020,wang2021}, or with a probability proportional to the arm's merit, known as meritocratic fairness~\citep{joseph2016,joseph2018}. While \EgalMAB{} emphasizes fairness among users, the aforementioned works deal with fairness among arms. However, there are alternative approaches that maintain a focus on fairness among users. One such approach involves maximizing the Nash social welfare (NSW) function. This function is defined as the product of rewards obtained by all users, which intuitively encodes the notion of fairness as it increases only when most users achieve high rewards. \citet{hossain2021} investigated a scenario involving $U$ users and $K$ arms, in which each arm's reward varies for each user due to different perceived utilities. In each time step, only one arm is played, and all users obtain rewards according to their perceived utilities. Their primary objective is to maximize the NSW. \citet{sawarni2023} examined an alternative fairness framework within the context of linear bandits. In their scenario, a new user is introduced at each time step, and fairness is ensured by considering the product of rewards across all time steps. Both of these works assume that only a \emph{single} arm is played at each time step, while the agent in \EgalMAB{} simultaneously selects \textit{multiple} arms. Additionally, instead of aiming to maximize the \textit{NSW function}, \EgalMAB{} focuses on maximizing the \textit{cumulative reward of the worst-off user}.

\begin{figure*}[t]
    \centering
    \includegraphics[width=0.7\textwidth]{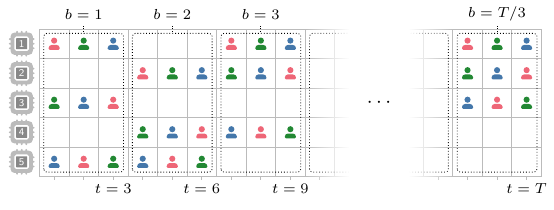}
    \caption{A trace of \EgalUCB{} with $K=5$ arms and $U=3$ users. When $b=2$, the three arms with the highest $\UCB_{a,b}$ values are $a\in\{2,4,5\}$, which are then assigned in a round-robin fashion across time steps $t\in\{4,5,6\}$. As $b$ increases, the estimates for $\mu_{a}$ for all arms $a$ improves. By $b=T/3$ for large $T$, the three arms with the highest $\UCB_{a,b}$ values are most likely $a\in\{1,2,3\}$, which are then assigned over time steps $t\in\{T-2,T-1,T\}$.
    }
    \label{fig:egalucb-illustrated}
\end{figure*}


\section{\EgalMAB{} Problem}
\label{sec:problem}

In this section, we will formally define the components involved in an \EgalMAB{} problem. We will start by presenting a working definition for each component. When relevant, we will supplement these definitions with the measure-theoretic details necessary for the proofs of lower bounds.

\paragraph{Environment.} Let $K$ be the number of arms, $U<K$ be the number of users, and $T$ be the time horizon. For each arm $a\in[K]$, let $p_{a}$ denote the reward density. An instance of the \EgalMAB{} problem is represented by the tuple $(\nu,U,T)$ where $\nu\coloneqq(p_{1},\dots,p_{K})$. When the context is clear, we also refer to $\nu$ as an \EgalMAB{} instance. We assume that the expected reward $\mu_{a}\coloneqq\E_{X\sim p_{a}}[X]$ obtained for playing arm $a$ is finite. Moreover, for convenience, we assume that the arms are indexed in such a way that $\mu_{1}\geq\cdots\geq\mu_{K}$. However, the agent is unaware of this ordering. 

More formally, for each arm $a\in[K]$, let $\P_{a}$ be the probability law for the reward obtained after playing arm $a$. For any $\P_{a}$-measurable set $B$, we have $\P_{a}(B)=\int_{B}p_{a}\,\dd\lambda$, where $\lambda$ is the Lebesgue measure and $p_{a}=\dd\P_{a}/\dd\lambda$ is the Radon--Nikodym derivative for $\P_{a}$. 

\paragraph{Agent Policy.} A solution to the \EgalMAB{} problem is characterized by an agent policy $\pi\coloneqq(\pi_{t})_{t}$. During each time step $t\in[T]$, the map $\pi_{t}$ considers the actions and rewards history (which we will define shortly, after introducing relevant notations) and assigns each user $u\in[U]$ to an arm $A_{u,t}\in[K]$ such that no two users are assigned the same arm. Subsequently, each user $u$ receives a reward $X_{u,t}$ drawn independently from the density $p_{A_{u,t}}$. Using these notations, we will denote the history that $\pi_{t}$ considers as $(A_{1},X_{1},\dots,A_{t-1},X_{t-1})$ where $A_{t}\coloneqq(A_{u,t})_{u}$ and $X_{t}\coloneqq(X_{u,t})_{u}$.

More formally, let $\powerset(S)$ denote the power set of a set $S$. For each time step $t\in[T]$, the stochastic map
\begin{equation*}
    \pi_{t}:(([K]\times\R)^{U})^{t-1}\times\powerset([K]^{U})\rightarrow[0,1]
\end{equation*}
is a probability kernel. We denote $\P_{\pi\nu}$ as the probability law for the interaction between the policy $\pi$ and the \EgalMAB{} instance $\nu$ over $T$ time steps. Thus, the density of $\P_{\pi\nu}$ is
\begin{align*}
    p_{\pi\nu}&(A_{1},X_{1},\dots,A_{T},X_{T})
    =
    \prod_{t=1}^{T}\pi_{t}(A_{t}|A_{1},X_{1},\dots,A_{t-1},X_{t-1})\prod_{u=1}^{U}p_{A_{u,t}}(X_{u,t}).
\end{align*}

\paragraph{Egalitarian Objective.} To achieve egalitarian fairness, we want to design a policy $\pi$ that maximizes the expected cumulative reward of the least-rewarded user. Let
\begin{equation*}
    S_{u,t}\coloneqq\sum_{s=1}^{t} X_{u,s}
\end{equation*}
be the cumulative reward for user $u$ up to time $t$. Formally, our egalitarian objective is to maximize $\min_{u\in[U]}\E[S_{u,T}]$. However, in line with most works in MAB, we will frame our problem as regret minimization. We define the \emph{expected cumulative regret} as
\begin{equation*}
    R_{T}\coloneqq \frac{T\mu_{*}}{U}-\min_{u\in[U]}\E[S_{u,T}]
\end{equation*}
where $\mu_{*}\coloneqq\mu_1+\dots+\mu_U$ is the sum of the expected reward of the top $U$ arms. The choice to compare with $T\mu_{*}/U$ is natural because the maximum expected cumulative reward obtained by the user with the least reward is at most $T\mu_{*}/U$, which is obtained by pulling the best arms in a round robin. To see this, observe that the sum of cumulative rewards for all $U$ users is $\sum_{u}\E[S_{u,T}]\leq T\mu_{*}$. Suppose, to the contrary, that $\min_{u}\E[S_{u,T}]>T\mu_{*}/U$, then $\sum_{u}\E[S_{u,T}]>T\mu_{*}$, which is a contradiction. 


\section{\EgalUCB{} Policy}
\label{sec:policy}

In this section, we describe our policy \EgalUCB{} that achieves near-optimal regret for the \EgalMAB{} problem. The \EgalUCB{} policy is based on the \UCBi{} policy \citep{auer2002} for solving classic MAB problems; however, it differs from the \UCBi{} in several key aspects. We present the pseudocode for \EgalUCB{} in Algorithm~\ref{alg:EgalUCB}, complemented by a visual guide in Figure~\ref{fig:egalucb-illustrated}. We also provide an alternate version of the pseudocode with more implementation details in Appendix~\ref{sec:detailed}.

\begin{algorithm}[t]
\caption{\EgalUCB}
\label{alg:EgalUCB}
\DontPrintSemicolon
\SetNoFillComment
let number of blocks $B_{a,0}=0$, cumulative reward $S_{a,0}=0$, and $\UCB_{a,0}=\infty$ for each $a\in[K]$\;
\For{$b=1,2,\dots,T/U$}{
    let $A_{b}\subseteqU[K]$ be a set of $U$ arms with highest $\UCB_{a,b-1}$\;
    play arms $A_{b}$ in a round robin fashion for the next $U$ steps\;
    update $B_{a,b}$ and $S_{a,bU}$ accordingly for each $a\in A_{b}$\;
    \vspace{1em}
    update $\UCB_{a,b}=\dfrac{S_{a,bU}}{B_{a,b}U}+\sqrt{\dfrac{6\ln(bU)}{B_{a,b}U}}$ for each $a\in A_{b}$
    \vspace{1em}\;
}
\end{algorithm}

\EgalUCB{} partitions the horizon into $B\coloneqq T/U$ blocks, each with $U$ steps. We assume, without loss of generality, that $T$ is divisible by $U$. In cases where this is not true, the difference in expected regret between the best and worst user is at most $(\mu_{1}+\dots+\mu_{\lceil U/2\rceil})-(\mu_{K-\lfloor U/2\rfloor}+\dots+\mu_{U})$, which is independent of $T$.

Let $A\subseteqU S$ denote that set $A$ is a subset of set $S$ with size $U$. \EgalUCB{} begins by initializing some statistics (Line 1). At the start of each block $b\in[B]$, it selects any set $A_{b}\subseteqU[K]$ consisting of the highest-ranked $U$ distinct arms as determined by their upper confidence bounds $\UCB_{a,b}$ (Line 3). Over a block with $U$ steps, these arms are then assigned to the $U$ users in a round-robin fashion (Line 4). After observing the rewards, \EgalUCB{} updates its statistics (Lines 5--6).

Since each user is assigned to every arm in $A_{b}$ exactly once, we have $\E[S_{u,T}]=\E[S_{u',T}]$ for all $u,u'\in[U]$. From a technical standpoint, this simplifies the regret by eliminating the $\min$ operator in the regret $R_{T}$:
\begin{equation*}
    R_{T}
    =\frac{T\mu_{*}}{U}-\min_{u\in[U]}\E[S_{u,T}]
    =\frac{T\mu_{*}}{U}-\E[S_{u,T}],\quad\forall u\in[U].
\end{equation*}


\section{Main Results}
\label{sec:mainresults}

In this section, we present our main theoretical results. We first state some necessary definitions and notations. Then, we present the regret upper bounds for the \EgalUCB{} policy. Following that, we discuss the policy-independent regret lower bound for \EgalMAB{}.

Let $X_{a,t}$ be the reward obtained from playing arm $a$ for the $t$-th time across all users. Note that if user $u$ plays arm $a$ at time step $t$, then the reward obtained is $X_{u,t}=X_{a,T_{a,t}}$ where $T_{a,t}$ is the number of times arm $a$ is played up till time step $t$. 

Let $B_{a,b}$ be the number of blocks that arm $a$ is played up till block $b$. Note that $B_{a,b}=T_{a,bU}$. Furthermore, let
\begin{equation*}
    \hat{\mu}_{a,b}\coloneqq\frac{1}{bU}\sum_{t=1}^{bU}X_{a,t}
\end{equation*}
be the empirical estimate of $\mu_{a}$ after playing arm $a$ for $b$ blocks. As the policy observe more rewards, it gains confidence about its estimate of $\mu_{a}$. This level of confidence is captured by
\begin{equation*}
    \epsilon_{b,b'}\coloneqq\sqrt{\frac{6\ln(bU)}{b'U}},
\end{equation*}
which is the confidence radius of playing an arm for $b'$ blocks after block $b$.

Denote $A_{*}\coloneqq[U]\subseteqU[K]$ as the set of best $U$ arms. For any set of arms $A\subseteq[K]$ of not necessarily size $U$, let $\Aminus\coloneqq A\backslash A_{*}$ be the set $A$ with the best $U$ arms removed. Furthermore, let the expected reward of $A$ be
\begin{equation*}
    \mu_{A}\coloneqq\sum_{a\in A}\mu_{a}.
\end{equation*}

The sub-optimality gap of $A$ is defined as $\Delta_{A}\coloneqq\mu_{A_{*}}-\mu_{A}$. In the extreme cases, the maximum sub-optimality gap
\begin{equation*}
    \Delta_{\max}\coloneqq(\mu_{1}+\dots+\mu_{U})-(\mu_{K-U+1}+\dots+\mu_K)
\end{equation*}
is obtained by selecting the worst set of $U$ arms, and the minimum non-zero sub-optimality gap $\Delta_{\min}\coloneqq\mu_{U}-\mu_{U+1}$ is obtained by replacing arm $U$ in $A_{*}$ with arm $U+1$, assuming that $\mu_{U}\neq\mu_{U+1}$.

Let $\nu=(p_{1},\dots,p_{K})$ be an instance of \EgalMAB{}. We say that $\nu$ is a $1$-subgaussian \EgalMAB{} if for all arms $a\in[K]$, $X\sim p_{a}$ is a $1$-subgaussian random variable. Theorems~\ref{thm:upperbound-dependent}~and~\ref{thm:upperbound-independent} respectively provide problem-dependent and problem-independent upper bounds for the expected cumulative regret of running \EgalUCB{} on a $1$-subgaussian \EgalMAB{}.

\begin{restatable}[Problem-Dependent Upper Bound]{theorem}{upperbounddependent}
\label{thm:upperbound-dependent}
Let $(\nu,T,U)$ a $1$-subgaussian \EgalMAB{}. After running \EgalUCB{} for $T$ time steps, we have
\begin{equation*}
    R_{T}\leq\frac{2136(K-U)\ln(T)}{\Deltamin}+\frac{4K\Deltamax}{U}.
\end{equation*}
\end{restatable}

\begin{restatable}[Problem-Independent Upper Bound]{theorem}{upperboundindependent}
\label{thm:upperbound-independent}
Let $(\nu,T,U)$ be a $1$-subgaussian \EgalMAB{} with $\mu_{a}\in[0,1]$ for all arms $a\in[K]$. After running \EgalUCB{} for $T$ time steps, we have
\begin{equation*}
    R_{T}\leq\sqrt{\frac{8544(K-U)\,T\ln(T)}{U}}+\frac{4K\min\{U,K-U\}}{U}.
\end{equation*}
\end{restatable}

When the number of arms $K$ and users $U$ are fixed, the problem-independent upper bound increases with the number of time steps $T$ at a rate of $O(\sqrt{T\ln(T)})$. Notably, when $U=1$, the \EgalMAB{} instance and the \EgalUCB{} policy reduce to the classic MAB instance and the \UCBi{} policy~\citep{auer2002}. Consequently, both the problem-dependent and problem-independent upper bounds can be reduced to the bound for classic \UCBi{} with minimal effort\footnote{Refer to the notes at the end of Lemma~\ref{lem:upperbound-part2} for the details.}.

Next, we examine how the number of users $U$ affects performance. Consider some fixed time horizon $T$ and number of arms $K$. Since $T\gg K$, the first terms in both the problem-independent and problem-dependent upper bounds dominates the regrets. Furthermore, when $K=U$, every user would have played every arm exactly once after each block, yielding an expected cumulative reward of $b\mu_{*}$ after block $b$. By definition, this implies that the expected cumulative regret $R_{T}=0$. This behavior is reflected in both upper bounds, since $K-U=0$ and $\Delta_{\max}=0$ when $K=U$.

Additionally, the problem-independent upper bound decreases as $U$ approaches $K$, and this reduction scales with $O(1/\sqrt{U})$. There are two reasons for this. Firstly, if we fix some \EgalMAB{} instance $\nu$ and vary $U$, then increasing $U$ results in decreasing $T\mu_{*}/U$. Secondly, since \EgalUCB{} assigns the arms in a round-robin fashion during each block, as long as the UCB values of the top arms are consistently among the highest regardless of their order, \EgalUCB{} will also consistently select a good set of arms. This implies that as $U$ increases, this problem becomes more statistically robust to the variability inherent in the estimates of the arms. Loosely speaking, the more users we have, the easier it is to match the performance of a policy always plays round-robin the set of arms $A_{*}$.

We further consider the scope for algorithmic improvement by deriving a policy-independent lower bound. Let $\nu=(p_1,\dots,p_K)$ be an \EgalMAB{} instance and $\pi$ be any policy. We denote $R_{\pi\nu}$ as the expected cumulative regret of running $\pi$ on $\nu$ for $T$ time steps. Theorem~\ref{thm:lowerbound} provides a policy-independent lower bound for the regret $R_{\pi\nu}$. This bound applies to the class $\V$ of all Gaussian \EgalMAB{} instances $\nu=(p_{1},\dots,p_{K})$ where, for all $a\in[K]$, the reward density $p_{a}=\mathcal{N}(\mu_{a},1)$ and $\mu_{a}\in[0,1]$ .

\begin{restatable}[Policy-Independent Lower Bound]{theorem}{lowerbound}
\label{thm:lowerbound}
Suppose $K\geq 2U$. For any policy $\pi$, there exist an \EgalMAB{} instance $\nu\in\V$ with regret
\begin{equation*}
    R_{\pi\nu}\geq\frac{\sqrt{(K-U)\,T}}{76U}.
\end{equation*}
\end{restatable}

The lower bound suggests that \EgalUCB{} is tight in $T$ up to logarithmic factors. This factor is expected when using a UCB-based policy due to the choice of confidence radius $\epsilon_{b,b'}$. Drawing parallels to the MOSS policy~\citep{audibert2009} in classic $K$-armed MAB problems, we also conjecture that there exists a MOSS-based policy that can shave away the $\sqrt{\log T}$ term in the regret bound.

Furthermore, there is a multiplicative gap of $1/\sqrt{U}$ between the lower bound and the problem-independent upper bound. Our experimental results in Section~\ref{sec:experiments} suggest that our upper bound analysis is not tight, as we empirically observe the $O(1/U)$ behavior when running \EgalUCB{}.


\section{Regret Analysis}
\label{sec:analysis}

In this section, we provide a proof sketch for the main results in Section~\ref{sec:mainresults}. Our analysis relies on techniques developed in the combinatorial semi-bandits literature~\citep{kveton2015a}. Detailed proofs for the upper bounds and lower bound can be found in Appendices~\ref{sec:upperproofs} and \ref{sec:lowerproofs} respectively.

\subsection{Problem-Dependent Upper Bound}
\label{ssec:upperbound}

The regret upper bound in Theorem~\ref{thm:upperbound-dependent} consists of a sum of two terms: the first term is dependent on $T$ and the second is independent of $T$. The first (resp. second) term arises from the regret accumulated over blocks where some \emph{good} event $\EventE_{b}$ occurs (resp. did not occur). This \emph{good} event $\EventE_{b}$ is the event that $\mu_{a}$ and its estimate $\hat{\mu}_{a}$ are at a distance of at most $\epsilon_{b-1,B_{a,b-1}}$ at the beginning of block $b$ for all $a\in[K]$. Formally, we define
\begin{equation*}
    \EventE_{b}\coloneqq\bigcap_{a=1}^{K}\,\Bigl\{\bigl|\hat{\mu}_{a,B_{a,b-1}}-\mu_{a}\bigr|\leq\epsilon_{b-1,B_{a,b-1}}\Bigr\}.
\end{equation*}
Note that we will often abuse the set notation when defining events. In particular, when we have a proposition $P$, we use $\EventE=\{P\}$ to signify that $\EventE$ is the set of all outcomes in the underlying probability space where $P$ holds. Lemma~\ref{lem:concentration} shows that $\EventE_{b}$ occurs with high probability.

\begin{restatable}{lemma}{concentration}
\label{lem:concentration}
Let $(\nu,T,U)$ be a $1$-subgaussian \EgalMAB{}. Then, for all blocks $b\in[B]$,
\begin{equation*}
    \P(\EventE_{b}^{\comp})\leq\frac{2K}{b^{2}U^{3}}.
\end{equation*}
\end{restatable}

To facilitate analysis, we then consider another \emph{good} event $\EventF_{b}$ for which the set of arms $A_{b}\subseteqU[K]$ selected during block $b$ is sub-optimal but ``not too bad''. This is defined as
\begin{equation*}
    \EventF_{b}\coloneqq\Biggl\{0<\Delta_{A_{b}}\leq2\sum_{\mathclap{a\in\Aminus_{b}}}\epsilon_{B,B_{a,b-1}}\Biggr\}.
\end{equation*}
Lemma~\ref{lem:eimpliesf} shows that if the set of arms $A_{b}$ played during block $b$ is sub-optimal and $\EventE_{b}$ occurs, then $\EventF_{b}$ must follow.

\begin{restatable}{lemma}{eimpliesf}
\label{lem:eimpliesf}
Let $b\in[B]$. If the set of arms $A_{b}$ played during block $b$ is sub-optimal and $\EventE_{b}$ occurs, then $\EventF_{b}$ also occurs.
\end{restatable}

We proceed by partitioning the regret into two terms: one conditioned on the high-probability event $\EventF_{b}$ using Lemma~\ref{lem:eimpliesf} and another conditioned on the low-probability event $\EventE_{b}^{\comp}$ using Lemma~\ref{lem:concentration}. This result is formally stated in Lemma~\ref{lem:upperbound-part1}.

\begin{restatable}{lemma}{upperboundi}
\label{lem:upperbound-part1}
Let $(\nu,T,U)$ be a $1$-subgaussian \EgalMAB{}. Then, after $T$ time steps, for all users $u\in[U]$, we have
\begin{equation*}
    R_{u,T}\leq\sum_{b=1}^{B}\E[\Delta_{A_{b}}\I\{\EventF_{b}\}]+\frac{\pi^{2}K\Deltamax}{3U^{3}}.
\end{equation*}
\end{restatable}

Let us now focus on bounding the contributions made by the high-probability term (i.e., the first term above). We first partition $\EventF_{b}$ into countably many mutually exclusive events $\{\EventG_{b,i}\}_{i}$ so that we can write
\begin{equation*}
    \Delta_{A_{b}}\I\{\EventF_{b}\}=\sum_{i=1}^{\infty}\Delta_{A_{b}}\I\{\EventG_{b,i}\},\qquad \mbox{almost surely}.
\end{equation*}
To define the events $\{\EventG_{b,i}\}_{i}$, let $\alpha=0.13$, $\beta=0.22$, and
\begin{equation*}
    \gamma=24\,\biggl(\frac{1-\beta}{\sqrt{\alpha}-\beta}\biggr)^{2}
\end{equation*}
be constants that are carefully chosen, and define
\begin{equation*}
    m_{b,i}\coloneqq\frac{\gamma\alpha^{i}U\ln(BU)}{\Delta_{A_{b}}^{2}}
\end{equation*}
for all $i\in\N$. Let
\begin{equation*}
    L_{b,i}\coloneqq\{a\in\Aminus_{b}:B_{a,b-1}\leq m_{b,i}\}
\end{equation*}
to be the set of arms in $\Aminus_{b}$ that are played for fewer than $m_{b,i}$ blocks at the beginning of block $b$. For convenience, let $L_{b,0}=\Aminus_{b}$. For each $b\in[B]$ and $i\in\N$, the event $\EventG_{b,i}$ is then defined as
\begin{equation*}
    \EventG_{b,i}\coloneqq
    \bigcap_{j=1}^{i-1}\bigl\{|L_{b,j}|<\beta^{j}U\bigr\}
    \cap
    \bigl\{|L_{b,i}|\geq\beta^{i}U \bigr\}
    \cap
    \bigl\{\Delta_{A_{b}}>0 \bigr\}.
\end{equation*}

It is clear that at most one of $\{\EventG_{b,i}\}_{i}$ can occur. However, to show that it is a partition for $\EventF_{b}$, we also need to show that at least one of $\{\EventG_{b,i}\}_{i}$ must happen. This is shown in Lemma~\ref{lem:fimpliesg}.

\begin{restatable}{lemma}{fimpliesg}
\label{lem:fimpliesg}
Assume that $b>K/U$. On the event $\EventF_{b}$, exactly one of the events in $\{\EventG_{b,i}\}_{i}$ occurs.
\end{restatable}

Using the newly-defined events $\{\EventG_{b,i}\}_{b,i}$, we bound the contributions of the the high-probability term in Lemma~\ref{lem:upperbound-part2}. The intuition behind the events $\{\EventG_{b,i}\}_{b,i}$, which is a common construction used in the proof of combinatorial semi-bandits \cite{kveton2015a}, is that it serves to upper bound the high probability term in the regret by the number of times the set of arms $A_b$ is played. This allows us to introduce the reciprocal of the gap term for individual arms $\Delta_{a,N_a}$ which then serves to derive a meaningful problem-dependent bound on the regret. 

\begin{restatable}{lemma}{upperboundii}
\label{lem:upperbound-part2}
Let $\nu=(p_{1},\dots,p_{K})$. Suppose that $p_{a}$ is the density for a $1$-subgaussian distribution for all $a\in[K]$. Then, after $T$ time steps,
\begin{align*}
    \sum_{b=1}^{B}\E[\Delta_{A_{b}}\I\{\EventF_{b}\}]
    &=\sum_{i=1}^{\infty}\sum_{b=b_{0}}^{B}\Delta_{A_{b}}\I\{\EventG_{b,i}\}+\sum_{b=1}^{b_{0}-1}\Delta_{A_{b}}\I\{\EventF_{b}\}\\
    &\leq2136\ln(BU)\sum_{a\in\Lambda}\frac{1}{\Delta_{a,N_{a}}}+\frac{K\Deltamax}{U}\\
    &\leq\frac{2136(K-U)\ln(BU)}{\Deltamin}+\frac{K\Deltamax}{U}.
\end{align*}
for all users $u\in[U]$.
\end{restatable}

The proof of Theorem~\ref{thm:upperbound-dependent} simply involves substituting the result of \ref{lem:upperbound-part2} into Lemma~\ref{lem:upperbound-part1}.

\begin{proof}[Proof of Theorem~\ref{thm:upperbound-dependent}]
We can bound the regret by
\begin{align*}
    R_{u,T}
    \leq\sum_{b=1}^{B}\E[\Delta_{A_{b}}\I\{\EventF_{b}\}]+\frac{\pi^{2}K\Deltamax}{3U^{3}}
    &\leq\frac{2136(K-U)\ln(BU)}{\Deltamin}+\frac{K\Deltamax}{U}+\frac{\pi^{2}K\Deltamax}{3U^{3}}\\
    &\leq\frac{2136(K-U)\ln(T)}{\Deltamin}+\frac{4K\Deltamax}{U}
\end{align*}
where the first inequality holds due to Lemma~\ref{lem:upperbound-part1} and the second inequality holds due to Lemma~\ref{lem:upperbound-part2}.
\end{proof}

\subsection{Problem-Independent Upper Bound}
\label{ssec:upperboundgapfree}

Much like the expression in Theorem~\ref{thm:upperbound-dependent}, the first term of the regret bound in Theorem~\ref{thm:upperbound-independent} arises from the cumulative regret accumulated in blocks where the high-probability event $\EventE_{b}$ occurs. Conversely, the second term arises from the regret accumulated in the initial blocks and the blocks where $\EventE_{b}$ did not occur. The proof of Theorem~\ref{thm:upperbound-independent} can be found in Appendix~\ref{sec:upperproofs}.

\subsection{Policy-Independent Lower Bound}

Our proof for Theorem~\ref{thm:lowerbound} consists of constructing two \EgalMAB{} instances $\nu$ and $\nu'$ that are ``close'' enough that it is difficult for any policy to distinguish between them statistically, yet ``far'' enough that a sequence of actions that is good for one instance is bad for the other. Specifically, let 
\begin{equation*}
    \Delta=\sqrt{\frac{K-U}{8TU^2}}
\end{equation*}
and set $\nu=(p_{1},\dots,p_{K})\in\V$ to be an \EgalMAB{} instance with
\begin{equation}
    \mu_{a}=\begin{cases}
        \Delta, & \text{if $a\in [U]$} \\
        0, & \text{otherwise.}
    \end{cases}
    \label{eq:mu}
\end{equation}
Assuming that $K\geq2U$, let
\begin{equation*}
    A'=\argmin_{A\subseteqU[K]\backslash[U]}\sum_{a\in A}\E_{\pi\nu}[T_{a,T}]
\end{equation*}
be the set of $U$ arms that are sub-optimal under $\nu$ have been played the fewest number of times under the distribution $\P_{\pi\nu}$. Set $\nu'\in\V$ to be an \EgalMAB{} instance with
\begin{equation}
    \mu'_{a}=\begin{cases}
        2\Delta, & \text{if $a\in A'$} \\
        \mu_{a}, & \text{otherwise}.
    \end{cases}
    \label{eq:muprime}
\end{equation}

Let $\BoGamma$ be the set of size-$U$ subsets of $[K]$ that contains at least $U/2$ arms from $[U]$, and let
\begin{equation*}
    \EventH\coloneqq\Biggl\{\sum_{A\in\BoGamma}T_{A,T}\leq\frac{T}{2}\Biggr\}
\end{equation*}
be the event that, for at least $T/2$ time steps, the policy $\pi$ selects a set in which at least half of it consists of arms from $[U]$. Lemma~\ref{lem:lowerbound-part1} uses the Bretagnolle--Huber inequality with $\EventH$ to reduce the lower bound computation to evaluating the KL-divergence between $\P_{\pi\nu}$ and $\P_{\pi\nu'}$.

\begin{restatable}{lemma}{lowerboundi}
\label{lem:lowerbound-part1}
Let $\nu$ and $\nu'$ be the \EgalMAB{} instances defined by \eqref{eq:mu} and \eqref{eq:muprime}. Under the assumptions of Theorem~\ref{thm:lowerbound}, we have
\begin{align*}
    R_{\pi\nu}+R_{\pi\nu'}
    &>\frac{\Delta T}{4}\bigl(\P_{\pi\nu}(\EventH)+\P_{\pi\nu'}(\EventH^{\comp})\bigr)\\
    &\geq\frac{\Delta T}{8}\exp\bigl(-\KL(\P_{\pi\nu}\|\P_{\pi\nu'})\bigr).
\end{align*}
\end{restatable}

Part of the novelty in our analysis involves reducing the computation of this KL-divergence to a combinatorial problem. Lemma~\ref{lem:kldecompose} decompose the KL-divergence from Lemma~\ref{lem:lowerbound-part1} into an expression that involves counting the number of times each arm $a\in A'$ is played, and Lemma~\ref{lem:tatcombi} shows, using a counting argument, that the number of times an arm $a\in A'$ is played is at most $TU^2/(K-U)$.

\begin{restatable}{lemma}{kldecompose}
\label{lem:kldecompose}
Let $\nu$ and $\nu'$ be the \EgalMAB{} instances defined by \eqref{eq:mu} and \eqref{eq:muprime}. Under the assumptions of Theorem~\ref{thm:lowerbound}, we have
\begin{equation*}
    \KL(\P_{\pi\nu}\|\P_{\pi\nu'})
    \leq4\Delta^{2}\sum_{a'\in A'}\E_{\pi\nu}[T_{a',T}].
\end{equation*}
\end{restatable}

\begin{restatable}{lemma}{tatcombi}
\label{lem:tatcombi}
Let $\nu$ and $\nu'$ be the \EgalMAB{} instances defined in \eqref{eq:mu} and \eqref{eq:muprime}. Under the assumptions of Theorem~\ref{thm:lowerbound}, we have
\begin{equation*}
    \sum_{a\in A'}T_{a,T}\leq\frac{TU^{2}}{K-U}
\end{equation*}
almost surely.
\end{restatable}

The proof of Theorem~\ref{thm:lowerbound} simply involves substituting the results of Lemma~\ref{lem:kldecompose} and Lemma~\ref{lem:tatcombi} into Lemma~\ref{lem:lowerbound-part1}.

\begin{proof}[Proof of Theorem~\ref{thm:lowerbound}]
We have
\begin{align*}
R_{T,\pi,\nu}+R_{T,\pi,\nu'}
&\geq\frac{\Delta T}{8}\exp\Biggl(-4\Delta^{2}\sum_{a'\in A'}\E_{\pi\nu}[T_{a',T}]\Biggr)\geq\frac{\Delta T}{8}\exp\Biggl(-\frac{4\Delta^2TU^2}{K-U}\Biggr)\\
&=\frac{\Delta T}{8}\exp(-1/2)>\frac{\sqrt{T(K-U)}}{38U}.
\end{align*}
Since $2\max\{R_{T,\pi,\nu},R_{T,\pi,\nu'}\}\geq R_{T,\pi,\nu}+R_{T,\pi,\nu'}$, dividing by $2$ concludes the proof.
\end{proof}


\section{Experiments}
\label{sec:experiments}

In this section, we present the results of our numerical experiments to validate the analysis of \EgalUCB{}. These experiments include both a synthetic environment (Section~\ref{sec:synthetic}) and real-world datasets (Sections~\ref{sec:clusterusage}~and~\ref{sec:movielens}). The code for these experiments can be found in the supplementary materials.

\subsection{Synthetic Experiments}
\label{sec:synthetic}

\begin{figure*}[t]
    \centering
    \includegraphics[width=\textwidth]{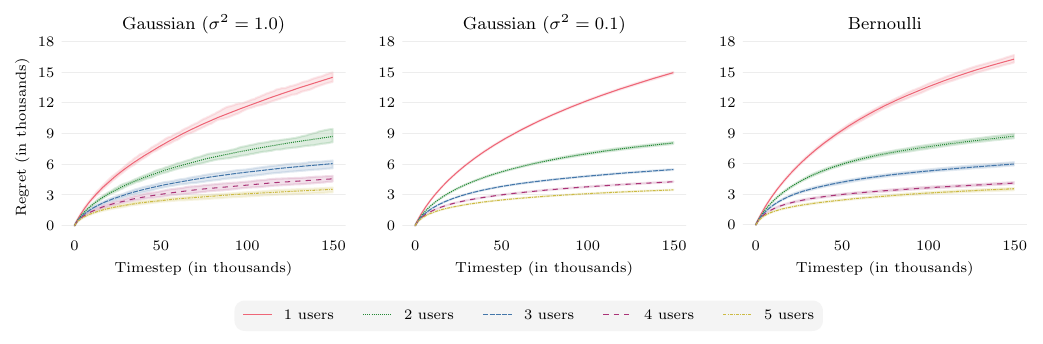}
    \caption{Expected regret incurred by \EgalUCB{} over $T=\num{150000}$ time steps on simulated data with $K=10$. Each line corresponds to a different $U$. The lighter region around  each line represents the range between the minimum and maximum expected regrets observed over a total of 30 independent runs.}
    \label{fig:regret-simulation}
\end{figure*}

To empirically verify the problem-independent upper bound in Theorem~\ref{thm:upperbound-independent}, we conducted experiments on three synthetic datasets: Gaussian bandits with variance $\sigma^{2}=1.0$, Gaussian bandits with variance $\sigma^{2}=0.1$ and Bernoulli bandits.

In each environment, we fixed $K=10$ and varied $U$ from $1$~to~$5$. For each choice of $U$, we ran the experiment $30$ times. In each run, we randomly generated the ground truth expected reward $\mu_{a}$ for each arm $a$ using a uniform distribution with support $[0.01,0.99]$. Figure~\ref{fig:regret-simulation} shows the expected regret incurred by \EgalUCB{} over $T=\num{150000}$ time steps. As predicted by Theorem~\ref{thm:upperbound-independent}, the expected regret $R_{T}$ is sub-linear in $T$ and diminishes with $U$.

\begin{figure}[t]
    \begin{minipage}[t]{0.475\columnwidth}
        \centering
        \includegraphics[height=120pt]{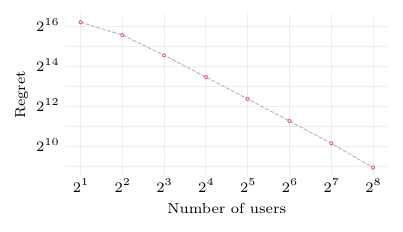}
        \caption{Expected regret incurred by \EgalUCB{} over $T=2^{18}$ time steps on Bernoulli bandits with $K=2^{10}$ arms.}
        \label{fig:user-bernoulli-v1}
    \end{minipage}\hfill
    \begin{minipage}[t]{0.475\columnwidth}
        \centering
        \includegraphics[height=120pt]{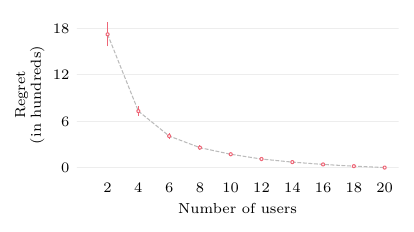}
        \caption{Expected regret incurred by \EgalUCB{} over $T=\num{126000}$ time steps on Bernoulli bandits with $K=20$ arms.}
        \label{fig:user-bernoulli-v2}
    \end{minipage}
\end{figure}

We conducted another experiment to verify the rate at which $R_{T}$ diminishes with $U$ when $K\gg U$. We fixed $K=2^{10}$ and $T=2^{18}$ while varying $U\in\{2^{1},\dots,2^{8}\}$. We ran \EgalUCB{} on instances of Bernoulli \EgalMAB{} where $\mu_{a}=0.8$ if $a\in[U]$ and $0.5$ otherwise. This choice of $\mu_{a}$ ensures that $\mu_{*}$ is kept constant for all choices of $U$. Figure~\ref{fig:user-bernoulli-v1} shows the log-log plot of $R_{T}$ against $U$. We observe that $R_{T}$ diminishes with $U$ at a rate of $O(U^{-c})$ for some constant $c\approx1.0$. This corroborates with our policy-independent lower bound in Theorem~\ref{thm:lowerbound} and suggests that our problem-independent upper bound in Theorem~\ref{thm:upperbound-independent} may be loose.

To assess the rate at which $R_{T}$ diminishes with $U$ when $U\to K$, we ran a similar experiment with $T=\num{126000}$, $K=20$, and $U\in\{2,4,\dots,18,20\}$ on Bernoulli \EgalMAB{} instances. For each $U$, we ran the experiment $30$ times. Figure~\ref{fig:user-bernoulli-v2} shows the plot of $R_{T}$ against $U$. We observe that as $U$ approaches $K$, the regret decreases to $0$, and specifically when $U=K$, the regret is exactly $0$. This observation aligns with the problem-independent upper bound in Theorem~\ref{thm:upperbound-independent}.

\subsection{Google Cluster Usage Trace Dataset}
\label{sec:clusterusage}

\begin{figure}[t]
    \begin{minipage}[t]{0.475\columnwidth}
        \centering
        \includegraphics[height=120pt]{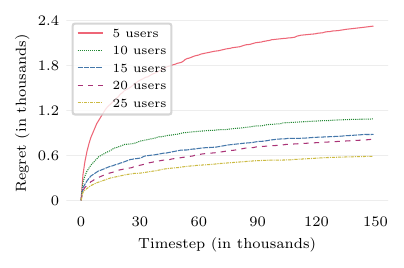}
        \caption{Expected regret incurred by \EgalUCB{} over $T=\num{150000}$ time steps on the Google Cluster Usage Trace dataset with $K=100$. Each line is associated with a different number of users $U$.}
        \label{fig:clusterusage}
    \end{minipage}\hfill
    \begin{minipage}[t]{0.475\columnwidth}
        \centering
        \includegraphics[height=120pt]{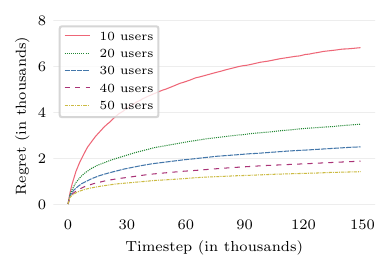}
        \caption{Expected regret incurred by \EgalUCB{} over $T=\num{150000}$ time steps on the MovieLens 25M dataset with $K=500$. Each line is associated with a different number of users $U$.}
        \label{fig:movielens}
    \end{minipage}
\end{figure}

The Google Cluster Usage Traces dataset comprises 2.4 TiB of compressed traces that record the workloads executed on Google compute cells~\citep{wilkes2020}. These traces are organized into tables that contain information about the machines and the instances running on them.

In our experiment, we focus on the \textit{InstanceUsage} table from the \texttt{clusterdata\_2019} trace. This table contains traces of both processor and memory usage during instance execution. To adapt to the \EgalMAB{} setting, we designate each arm $a$ as a machine, uniquely identifiable using the \texttt{machine\_id} field in the table. We implicitly construct its reward distribution $\P_{a}$ by drawing an entry uniformly from the trace that corresponds to the machine and return the negative of the \texttt{cycles\_per\_instruction} field for its reward. Given the substantial size of the dataset, we will limit our analysis to the initial $4$ million entries from the table. We pick $K=100$ machines that contain the most amount of trace entries. Then, we formulate scenarios involving $U\in\{5,10,15,20,25\}$ \emph{unseen} users. At each time step $t\in[T]$ where $T=\num{150000}$, we employ the \EgalUCB{} policy to assign machines to these users.

Figure~\ref{fig:clusterusage} shows the expected regret $R_{T}$ over time. Since the distribution $\P_{a}$ has bounded support, it is a subgaussian distribution. This empirical finding supports the $O(\sqrt{T\ln(T)})$ growth rate that Theorem~\ref{thm:upperbound-independent} predicts.

\subsection{MovieLens 25M Dataset}
\label{sec:movielens}

The MovieLens 25M dataset is widely used in recommender systems~\citep{kuzelewska2014,forouzandeh2021} and collaborative filtering~\citep{he2017,gonzalez2022} research. This dataset encompasses a substantial collection of $\num{25000000}$ user ratings contributed by \num{162000} users for a repository of \num{62000} movies.

To adapt to the \EgalMAB{} setting, we randomly select $K=500$ movies and treat them as arms. For each movie $a\in[K]$, we implicitly construct its reward distribution $\P_{a}$ using the user ratings provided by existing users in the dataset. This empirical distribution is categorical and has support residing within $\{0.5,1.0,\dots,4.5,5.0\}$. Then, we formulate a scenario involving $U\in\{10,20,30,40,50\}$ \emph{unseen} users. At each time step $t\in[T]$ where $T=150,000$, we employ the \EgalUCB{} policy to assign movies to these users.

Figure~\ref{fig:movielens} shows the expected regret $R_{T}$ over time. Similar to the Google Cluster Usage Trace dataset, the distribution $\P_{a}$ is subgaussian for all machines $a$. As such, this empirical finding aligns with the $O(\sqrt{T\ln(T)})$ growth rate that Theorem~\ref{thm:upperbound-independent} predicts.


\section{Discussion}

In this work, we introduced \EgalMAB{}, an extension to the MAB framework with egalitarian considerations. The \EgalUCB{} policy was proposed and shown to achieve an expected regret of $\smash{O\bigl(\sqrt{T\ln(T)\cdot(K-U)\cdot U^{-1}}\bigr)}$.
We also derived a lower bound that matches the upper bound up to a multiplicative gap of $1/\sqrt{U}$ and a term logarithmic in $T$. Our experiments on simulated and real-world data validated the theoretical analysis. Our empirical results lead us to conjecture that \EgalUCB{} is indeed tight with respect to $U$. This gap could potentially be reconciled with a more refined analysis of the upper bound in future work. Other future works include:

\textbf{Adversarial semi-bandits.\quad} It is natural to consider the adversarial variant of our setup. After choosing a randomized assignment at each time step, an adaptive online adversary chooses the reward for each arm. Since that the set of all randomized assignments $\mathcal{B}$ (a.k.a.\ the Birkhoff polytope) is a convex set and the expected reward given a randomized assignment is a convex function, we conjecture that a modification of \texttt{Component Hedge}~\citep{koolen2010} and \texttt{PermELearn}~\citep{helmbold2009} can achieve an asymptotically near-optimal solution under the egalitarian consideration.

\textbf{Thompson Sampling.\quad} Another possible direction is to develop a Thompson sampling approach for the \EgalMAB{} problem, potentially by adapting Combinatorial Thompson Sampling~\citep{wang18}.

\textbf{Arms with capacity.\quad} Suppose that at each time step, we can assign at most $C$ users to each arm. When multiple users are assigned to the same arm at a given time step, we assume that they receive the same reward. This scenario particularizes to our setting when $C=1$. Since the optimal assignment is to round-robin the top $U/C$ arms, we can redefine the regret with $\mu_{*}=\mu_{1}+\dots+\mu_{U/C}$. We can show that a modified version of \EgalUCB{}, in which we split the horizon into $B=TC/U$ blocks and play the $U/C$ arms with the highest UCB value, achieves a regret of $R(T)\leq2136(K-U/C)\ln(T)\Deltamin^{-1}+KCU^{-1}\Deltamax$ where $\Deltamin$ and $\Deltamax$ are redefined by replacing $U$ by $U/C$. We conjecture that this bound may be improved by dividing the horizon into $C$ phases in which we eliminate all except $K/c$ arms in phase $c\in[C]$ and play the $U/c$ arms with the highest UCB values in a round-robin fashion. If one can prove that the top $U/c$ arms survive the elimination after each phase with high probability, it is possible to achieve a factor of $C^{-1}$ in the first term in the upper bound on $R(T)$.


\subsubsection*{Acknowledgments}

This research/project is supported by the National Research Foundation, Singapore under its AI Singapore Programme (AISG Award No: AISG2-PhD/2021-08-011). This research/project is also supported by a Ministry of Education Tier 2 grant under grant number A-9000423-00-00. 

\bibliography{main}
\bibliographystyle{tmlr}

\appendix
\include{supplementary}

\end{document}

%% file: math_commands.tex
\usepackage{amsmath,amsfonts,bm}

\def\eqref#1{equation~\ref{#1}}

\def\1{\bm{1}}

\DeclareMathAlphabet{\mathsfit}{\encodingdefault}{\sfdefault}{m}{sl}
\SetMathAlphabet{\mathsfit}{bold}{\encodingdefault}{\sfdefault}{bx}{n}

\newcommand{\E}{\mathbb{E}}

\newcommand{\R}{\mathbb{R}}

\newcommand{\KL}{D_{\mathrm{KL}}}

\DeclareMathOperator*{\argmin}{arg\,min}

\newcommand{\EgalUCB}{\texttt{EgalUCB}}
\newcommand{\EgalMAB}{\texttt{EgalMAB}}

\newcommand{\dd}{\mathrm{d}}
\newcommand{\I}{\mathbb{I}}
\newcommand{\N}{\mathbb{N}}
\renewcommand{\P}{\mathbb{P}}
\newcommand{\V}{\mathcal{V}}

\newcommand{\EventE}{\mathcal{E}}
\newcommand{\EventF}{\mathcal{F}}
\newcommand{\EventG}{\mathcal{G}}
\newcommand{\EventH}{\mathcal{H}}

\newcommand{\Aminus}{A^{-}}
\newcommand{\BoGamma}{\mathbf{\Gamma}}
\newcommand{\Deltamin}{\Delta_{\mathrm{min}}}
\newcommand{\Deltamax}{\Delta_{\mathrm{max}}}

\newcommand{\comp}{\mathsf{c}}
\newcommand{\powerset}{\mathcal{P}}
\newcommand{\ind}{\mathrm{ind}}
\newcommand{\subseteqU}{\subseteq_{\scriptscriptstyle{U}}}
\newcommand{\UCB}{\mathrm{UCB}}
\newcommand{\UCBi}{\texttt{UCB1}}

%% file: supplementary.tex
\onecolumn

\section{Notations}
\label{sec:notations}

\begin{table}[H]
    \centering
    \caption{Summary of notations}
	\label{tab:locations}
	\begin{tabular}{ll}
    \toprule
        $A_{b}$ & set of arms played during block $b$ \\
        $A_{t}$ & set of arms played during time step $t$ \\
		$A_{u,t}$ & arm played by user $u$ during time step $t$ \\
        $A_{*}$ & set of top $U$ arms \\
        $\Aminus$ & set of arms $A\backslash A_{*}$ \\
		$B$ & number of blocks \\
		$B_{a,b}$ & number of blocks arm $a$ is played up till block $b$ \\
		$K$ & number of arms \\
		$L_{b,i}$ & set of arms in $\Aminus_{b}$ where $B_{a,b-1}\leq m_{b,i}$ \\
        $N_{a}$ & number of sub-optimal sets played that includes $a$ \\
        $p_{a}$ & probability density of reward for arm $a$ \\
		$p_{\pi\nu}$ & probability density for the interaction between $\pi$ and $\nu$ \\
		$S_{u,t}$ & cumulative reward for user $u$ after $t$ time step \\
		$X_{t}$ & set of rewards obtained during time step $t$ \\
		$X_{a,t}$ & $t$-th reward obtained from playing arm $a$ \\
		$X_{u,t}$ & reward obtained by user $u$ during time step $t$ \\
		$R_{T}$ & regret after $T$ time steps \\
		$R_{\pi\nu}$ & regret of running $\pi$ on instance $\nu$ for $T$ time steps \\
		$T$ & time horizon \\
		$T_{a,t}$ & number of time steps arm $a$ is played up till time step $t$ \\
		$U$ & number of users \\
    \midrule
        $\EventE_{b}$ & event that $\mu_{a}$ is close to $\hat{\mu}_{a}$ for all $a\in[K]$ \\
        $\EventF_{b}$ & event that $A_{b}$ is sub-optimal but ``not too bad'' \\
        $\EventG_{b,i}$ & sub-event of $\EventF_{b}$ used in the regret analysis \\
        $\EventG_{b,i,a}$ & arm-dependent variant of $\EventG_{b,i}$ \\
        $\EventH$ & sets in $\BoGamma$ are played for at most $B/2$ times \\
        $\V$ & set of all 1-subgaussian \EgalMAB{} instances \\
    \midrule
        $\P_{a}$ & probability law of reward for arm $a$ \\
        $\P_{\pi\nu}$ & probability law for the interaction between $\pi$ and $\nu$ \\
    \midrule
        $\Delta_{A}$ & difference between $\mu_{*}$ and $\mu_{A}$ \\
        $\Deltamax$ & difference between $\mu_{U}$ and $\mu_{U+1}$ \\
        $\Deltamin$ & difference between $\mu_{1}+\dots+\mu_{U}$ and $\mu_{K-U+1}+\dots+\mu_{K}$ \\
        $\Lambda$ & set of arms $[K]\backslash A_{*}$ \\
    \midrule
        $\alpha$ & technical constant used for the proof \\
        $\beta$ & technical constant used for the proof \\
        $\epsilon_{b,b'}$ & confidence radius of playing $b'$ blocks after block $b$ \\
        $\gamma$ & technical constant used for the proof \\
        $\lambda$ & Lebesgue measure \\
        $\mu_{a}$ & expectation for distribution $\P_{a}$ \\
        $\mu_{A}$ & sum of the expected reward over the arms in $A$ \\
        $\mu_{*}$ & sum of the expected reward over the top $U$ arms \\
        $\hat{\mu}_{a,b}$ & empirical estimate of $\mu_{a}$ after playing $a$ for $b$ blocks \\
        $\nu$ & \EgalMAB{} instance \\
        $\pi$ & policy for \EgalMAB{} \\
        $\rho$ & counting measure \\
    \bottomrule
	\end{tabular}
\end{table}

\clearpage


\section{Detailed Algorithm}
\label{sec:detailed}

\begin{algorithm}[h]
\caption{\EgalUCB{} with implementation details}
\label{alg:EgalUCBDetailed}
\DontPrintSemicolon
initialize current block $b=0$\;
initialize current time step $t=0$\;
\ForEach{$a\in[K]$}{
    let number of blocks $B_{a,b}=0$\;
    let cumulative reward $S_{a,t}=0$\;
    let upper confidence bound $\UCB_{a,b}=\infty$\;
}
\While{$b\leq T/U$}{
    update $b=b+1$\;
    let $A_{b}\subseteqU[K]$ be a set of $U$ arms with highest $\UCB_{a,b-1}$\;
    let $\ind=(1,\dots,U)$\;
    \ForEach{$i\in[U]$}{
        update $t=t+1$\;
        \ForEach{$u\in[U]$}{
            let $A_{u,t}=A_{b}[\ind[u]]$\;
        }
        play $(A_{1,t},\dots,A_{U,t})$ and receive $(X_{1,t},\dots,X_{U,t})$\;
        \ForEach{$u\in[U]$}{
            let $a=A_{b}[\ind[u]]$\;
            let $S_{a,t}=S_{a,t-1}+X_{u,t}$\;
        }
        circular shift $\ind$ by one to the right
    }
    \ForEach{$a\in A_{b}$}{
        let $B_{a,b}=B_{a,b-1}+1$\;
    }
    \ForEach{$a\in[K]$}{
        let $\UCB_{a,b}=\dfrac{S_{a,t}}{B_{a,b}U}+\sqrt{\dfrac{6\ln(bU)}{B_{a,b}U}}$\;
    }
}
\end{algorithm}

\clearpage


\section{Proofs of Upper Bounds}
\label{sec:upperproofs}

This section contains the proof for the upper bounds.

\concentration*

\begin{proof}
Since the rewards are 1-subgaussian, we have
\begin{equation*}
    \P\bigl(|\hat{\mu}_{a,b'}-\mu_a|>\epsilon_{b,b'}\bigr)
    \leq2\exp\biggl(-\frac{1}{2}b'U\epsilon_{b,b'}^{2}\biggr)
\end{equation*}
for any $b'\in[b]$ due to Chernoff's bound. Then, by applying the union bound over all arms and all possible values of $B_{a,b-1}$, we have
\begin{align*}
    \P(\EventE_{b}^{\comp})
    &\leq\sum_{a=1}^{K}\sum_{b'=1}^{b}\P\bigl(|\hat{\mu}_{a,b'}-\mu_a|>\epsilon_{b,b'}\bigr)\\
    &\leq\sum_{a=1}^{K}\sum_{b'=1}^{b}2\exp\biggl(-\frac{1}{2}b'U\epsilon_{b,b'}^{2}\biggr)\\
    &\leq\sum_{a=1}^{K}\sum_{b'=1}^{b}\frac{2}{(bU)^3}\\
    &\leq\frac{2K}{b^2U^3}.\qedhere
\end{align*}
\end{proof}

\vspace{2em}

\eimpliesf*

\begin{proof}
Since $A_{b}$ is assumed to be sub-optimal, we already have $\Delta_{A_{b}}>0$. Denote $\Aminus_{*}$ to be the set $A_{*}\backslash A_{b}$. Observe that
\begin{equation*}
    \Delta_{A_{b}}
    =\sum_{a\in A_{*}}\mu_{a}-\sum_{a\in A_{b}}\mu_{a}
    =\sum_{a\in\Aminus_{*}}\mu_{a}-\sum_{a\in\Aminus_{b}}\mu_{a}.
\end{equation*}
since the terms that are associated to arms in $A_{b}\cap A_{*}$ cancel  out. Furthermore, since \EgalUCB{} chooses $A_{b}$ instead of $A_{*}$, we have
\begin{equation*}
    \sum_{a\in\Aminus_{b}}\hat{\mu}_{a,B_{a,b-1}}+\epsilon_{b-1,B_{a,b-1}}
    \geq\sum_{a\in\Aminus_{*}}\hat{\mu}_{a,B_{a,b-1}}+\epsilon_{b-1,B_{a,b-1}}.
\end{equation*}
Using these observations, we have
\begin{align*}
    \sum_{a\in\Aminus_{b}}\mu_{a}+2\epsilon_{b-1,B_{a,b-1}}
    &\geq\sum_{a\in\Aminus_{b}}\hat{\mu}_{a,B_{a,b-1}}+\epsilon_{b-1,B_{a,b-1}}\\
    &\geq\sum_{a\in\Aminus_{*}}\hat{\mu}_{a,B_{a,b-1}}+\epsilon_{b-1,B_{a,b-1}}\\
    &\geq\sum_{a\in\Aminus_{*}}\mu_{a}
\end{align*}
where the first and last inequality holds due to the assumption that $\EventE_{b}$ occurs. Rearranging this, we have
\begin{align*}
    \Delta_{A_{b}}
    =\sum_{a\in\Aminus_{*}}\mu_{a}-\sum_{a\in\Aminus_{b}}\mu_{a}
    &\leq2\sum_{a\in\Aminus_{b}}\epsilon_{b-1,B_{a,b-1}}\\
    &\leq2\sum_{a\in\Aminus_{b}}\epsilon_{B,B_{a,b-1}}
\end{align*}
where the last inequality holds because $b-1\leq B$.
\end{proof}

\upperboundi*

\begin{proof}
We begin by decomposing $R_{u,T}$ into
\begin{equation*}
    R_{u,T}=\frac{T\mu_{*}}{U}-\E[S_{u,T}]=\sum_{b=1}^{B}\mu_{*}-\E[\mu_{A_{b}}]=\sum_{b=1}^{B}\E[\Delta_{A_{b}}].
\end{equation*}
Since $\I\{\EventE\}+\I\{\EventE^{\comp}\}=1$ almost surely for any event $\EventE$, we can split
\begin{equation*}
    \sum_{b=1}^{B}\E[\Delta_{A_{b}}]
    =\underbrace{\sum_{b=1}^{B}\E[\Delta_{A_{b}}\I\{\EventE_{b}\}]}_{(\spadesuit)}
    +\underbrace{\sum_{b=1}^{B}\E[\Delta_{A_{b}}\I\{\EventE_{b}^{\comp}\}]}_{(\clubsuit)}.
\end{equation*}
To bound $(\spadesuit)$, we use $\smash{\I\{\Delta_{A_{b}}\!=0\}+\I\{\Delta_{A_{b}}\!>0\}}=1$ to get
\begin{equation*}
    \sum_{b=1}^{B}\E[\Delta_{A_{b}}\I\{\EventE_{b}\}\I\{\Delta_{A_{b}}\!=0\}]
    +\sum_{b=1}^{B}\E[\Delta_{A_{b}}\I\{\EventE_{b}\}\I\{\Delta_{A_{b}}\!>0\}].
\end{equation*}
Since $\smash{\Delta_{A_{b}}\I\{\Delta_{A_{b}}\!=0\}=0}$ almost surely, the first term is $0$. To deal with the second term, observe for any events $\EventE,\EventE',\EventF$, if $\EventE$ and $\EventE'$ implies $\EventF$, then $\I\{\EventE_{1}\cap\EventE_{2}\}=\I\{\EventE_{1}\}\I\{\EventE_{2}\}\leq\I\{\EventF\}$ almost surely. As such, we have
\begin{equation*}
    \sum_{b=1}^{B}\E[\Delta_{A_{b}}\I\{\EventE_{b}\}\I\{\Delta_{A_{b}}\!>0\}]
    \leq\sum_{b=1}^{B}\E[\Delta_{A_{b}}\I\{\EventF_{b}\}]
\end{equation*}
by Lemma~\ref{lem:eimpliesf}, thus concluding the proof for $(\spadesuit)$. To bound $(\clubsuit)$, observe that
\begin{equation*}
    \sum_{b=1}^{B}\E[\Delta_{A_{b}}\I\{\EventE_{b}^{\comp}\}]
    =\sum_{b=1}^{B}\E[\Delta_{A_{b}}|\EventE_{b}^{\comp}]\P(\EventE_{b}^{\comp}).
\end{equation*}
The expectation term can be bounded by
\begin{align*}
    \E[\Delta_{A_{b}}|\EventE_{b}^{\comp}]
    &=\sum_{A\subseteqU[K]}\P(A_{b}=A|\EventE_{b}^{\comp})\,\Delta_{A}\\
    &\leq\sum_{A\subseteqU[K]}\P(A_{b}=A|\EventE_{b}^{\comp})\,\Deltamax\\
    &=\Deltamax.
\end{align*}
Furthermore, we know that $\P(\EventE_{b}^{\comp})\leq2K/b^{2}U^{3}$ from Lemma~\ref{lem:concentration}. Substituting these results back, we have
\begin{align*}
    \sum_{b=1}^{B}\E[\Delta_{A_{b}}|\EventE_{b}^{\comp}]\P(\EventE_{b}^{\comp})
    &\leq\sum_{b=1}^{B}\Deltamax\cdot\frac{2K}{b^{2}U^{3}}\\
    &\leq\frac{2K\Deltamax}{U^{3}}\sum_{b=1}^{\infty}\frac{1}{b^{2}}\\
    &=\frac{\pi^{2}K\Deltamax}{3U^{3}},
\end{align*}
thus concluding the proof for $(\clubsuit)$.
\end{proof}

\fimpliesg*

\begin{proof}
It is clear by definition that at most one of $\{\EventG_{b,i}\}_{i}$ can happen. We are left to show that at least one of $\{\EventG_{b,i}\}_{i}$ must happen. Suppose that none of $\{\EventG_{b,i}\}_{i}$ happens. Thus $\smash{|L_{b,i}|<\beta^{i}U}$ for all $i\in\N$. First, we claim that all arms $a\in\Aminus_{b}$ are played at least once after block $K/U$. To see this, observe that the radius $\epsilon_{a,b}=\infty$ until $a$ is played; and after which $\epsilon_{a,b}<\infty$. This claim implies that there exist some sufficiently large $j\in\N$ such that the set $L_{b,j}=\emptyset$. Moreover, since $\{L_{b,i}\}_{i}$ is a non-increasing sequence of sets, we have that for all $b>K/U$, all arms $a\in\Aminus_{b}$ must lie in exactly one $L_{b,i-1}\backslash L_{b,i}$. Thus, we have
\begin{align*}
    \sum_{a\in\Aminus_{b}}\frac{1}{\sqrt{B_{a,b-1}}}
    &=\sum_{i=1}^{\infty}\sum_{a\in L_{b,i-1}\backslash L_{b,i}}\frac{1}{\sqrt{B_{a,b-1}}}\\
    &<\sum_{i=1}^{\infty}\sum_{a\in L_{b,i-1}\backslash L_{b,i}}\frac{1}{\sqrt{m_{b,i}}}\\
    &=\sum_{i=1}^{\infty}\frac{|L_{b,i-1}|-|L_{b,i}|}{\sqrt{m_{b,i}}}\\
    &=\frac{|L_{b,0}|}{\sqrt{m_{b,1}}}+\sum_{i=1}^{\infty}|L_{b,i}|\cdot\Biggl(\frac{1}{\sqrt{m_{b,i+1}}}-\frac{1}{\sqrt{m_{b,i}}}\Biggr)\\
    &<\frac{\beta^{0}U}{\sqrt{m_{b,1}}}+\sum_{i=1}^{\infty}\beta^{i}U\cdot\Biggl(\frac{1}{\sqrt{m_{b,i+1}}}-\frac{1}{\sqrt{m_{b,i}}}\Biggr)\\
    &=U\sum_{i=1}^{\infty}\frac{\beta^{i-1}-\beta^{i}}{\sqrt{m_{b,i}}}.
\end{align*}
Substituting $m_{b,i}$ into the inequality, we have
\begin{equation*}
    U\sum_{i=1}^{\infty}\frac{\beta^{i-1}-\beta^{i}}{\sqrt{m_{b,i}}}
    =U\sum_{i=1}^{\infty}\frac{\beta^{i-1}-\beta^{i}}{\sqrt{\gamma\alpha^{i}U\ln(BU)/\Delta_{A_{b}}^{2}}}
\end{equation*}
Rearranging the terms and evaluating the constants, we have
\begin{align*}
    U\sum_{i=1}^{\infty}\frac{\beta^{i-1}-\beta^{i}}{\sqrt{\gamma\alpha^{i}U\ln(BU)/\Delta_{A_{b}}^{2}}}
    &=\frac{1-\beta}{\beta}\sqrt{\frac{U}{\gamma\ln(BU)}}\sum_{i=1}^{\infty}\,\biggl(\frac{\beta}{\sqrt{\alpha}}\biggr)^{i}\cdot\Delta_{A_{b}}\\
    &<\sqrt{\frac{U}{24\ln(BU)}}\cdot\Delta_{A_{b}}.
\end{align*}
Since $\EventF_{b}$ happens, we have
\begin{equation*}
    \Delta_{A_{b}}
    \leq\sum_{a\in\Aminus_{b}}\sqrt{\frac{24\ln(BU)}{B_{a,b-1}U}}
    <\Delta_{A_{b}}
\end{equation*}
which is a contradiction. Hence, at least one of the events $\{\EventG_{b,i}\}_{i}$ must happen.
\end{proof}

\upperboundii*

\begin{proof}
For each arm $a\in[K]$, each index $i\in\mathbb{N}$, and each block $b\in[B]$, let
\begin{equation*}
\EventG_{b,i,a}=\EventG_{b,i}\cap\bigl\{a\in\Aminus_{b}\bigr\}\cap\bigl\{B_{a,b-1}\leq m_{b,i}\bigr\}
\end{equation*}
be an arm-dependent variant of the event $\EventG_{b,i}$. Since at least $\beta^{i}U$ arms satisfy $\EventG_{b,i,a}$ when $\EventG_{b,i}$ occurs, we have
\begin{equation*}
    \I\{\EventG_{b,i}\}\leq\frac{1}{\beta^{i}U}\sum_{a\in\Lambda}\I\{\EventG_{b,i,a}\}
\end{equation*}
where $\Lambda\coloneqq[K]\backslash A_{*}$ is the set of arms that are not in $A_{*}$. For each arm $a\in\Lambda$, let
\begin{equation*}
    N_{a}\coloneqq\bigl|\{A_{b}:\text{$b\in[B]$ and $a\in\Aminus_{b}$ and $\Delta_{A_{b}}\!>0$}\}\bigr|
\end{equation*}
be the number of distinct sub-optimal sets played that includes $a$, and let $A_{a,1},\dots,A_{a,N_{a}}$ be these sets sorted by non-ascending order of its sub-optimality gap. In other words, if we denote
\begin{equation*}
    \Delta_{a,j}\coloneqq\Delta_{A_{a,j}},
\end{equation*}
then $\Delta_{a,1}\geq\dots\geq\Delta_{a,N_{a}}$. Let $b_{0}=K/U+1$. Almost surely we have
\begin{align*}
    \sum_{b=b_{0}}^{B}\Delta_{A_{b}}\I\{\EventG_{b,i}\}
    &\leq\sum_{a\in\Lambda}\sum_{b=b_{0}}^{B}\frac{\Delta_{A_{b}}}{\beta^{i}U}\,\I\{\EventG_{b,i,a}\}\\
    &=\sum_{a\in\Lambda}\sum_{b=b_{0}}^{B}\sum_{j=1}^{N_{a}}\frac{\Delta_{a,j}}{\beta^{i}U}\,\I\{\EventG_{b,i,a}\}\I\{A_{b}=A_{a,j}\}\\
    &=\sum_{a\in\Lambda}\sum_{j=1}^{N_{a}}\frac{\Delta_{a,j}}{\beta^{i}U}\sum_{b=b_{0}}^{B}\,\I\{\EventG_{b,i,a}\}\I\{A_{b}=A_{a,j}\}
\end{align*}
We can upper bound this expression by considering the worst-case realization of the number of blocks each set $A_{a,j}$ is played. Let us start with $j=1$. Since $A_{a,1}$ has the largest gap $\Delta_{a,1}$, the worst case realization is when $A_{a,1}$ is played as many times as the event $\EventG_{b,i,a}$ allows. Recall that $\EventG_{b,i,a}$ implies that $B_{a,b-1}\leq m_{b,i}$. It follows that we can play $A_{a,1}$ for at most $\smash{\gamma\alpha^{i}U\ln(BU)/\Delta_{a,1}^{2}}$ blocks. We can use this argument to find the worst-case realization on the number of blocks $A_{a,2}$ is played. This works out to be
\begin{equation*}
    \frac{\gamma\alpha^{i}U\ln(BU)}{\Delta_{a,2}^{2}}-\frac{\gamma\alpha^{i}U\ln(BU)}{\Delta_{a,1}^{2}}
    =\gamma\alpha^{i}U\ln(BU)\cdot\Biggl(\frac{1}{\Delta_{a,2}^2}-\frac{1}{\Delta_{a,1}^2}\Biggr).
\end{equation*}
Repeating this argument, we have
\begin{align*}
    &\,\sum_{a\in\Lambda}\sum_{j=1}^{N_{a}}\frac{\Delta_{a,j}}{\beta^{i}U}\sum_{b=b_{0}}^{B}\,\I\{\EventG_{b,i,a}\}\I\{A_{b}=A_{a,j}\}\\
    \leq&\,\sum_{a\in\Lambda}\frac{\gamma\alpha^{i}\ln(BU)}{\beta^{i}}\cdot\Biggl(\frac{1}{\Delta_{a,1}}+\sum_{j=2}^{N_{a}}\Delta_{a,j}\,\Biggl(\frac{1}{\Delta_{a,j}^2}-\frac{1}{\Delta_{a,j-1}^2}\Biggr)\Biggr).
\end{align*}
The terms within the bracket can be further bounded by
\begin{align*}
    &\;\frac{1}{\Delta_{a,1}}+\sum_{j=2}^{N_{a}}\Delta_{a,j}\,\Biggl(\frac{1}{\Delta_{a,j}^2}-\frac{1}{\Delta_{a,j-1}^2}\Biggr)\\
    =&\;\frac{1}{\Delta_{a,N_{a}}}+\sum_{j=1}^{N_{a}-1}\frac{\Delta_{a,j}-\Delta_{a,j+1}}{\Delta_{a,j}^2}\\
    \leq&\;\frac{1}{\Delta_{a,N_{a}}}+\sum_{j=1}^{N_{a}-1}\frac{\Delta_{a,j}-\Delta_{a,j+1}}{\Delta_{a,j}\cdot\Delta_{a,j+1}}\\
    =&\;\frac{1}{\Delta_{a,N_{a}}}+\sum_{j=1}^{N_{a}-1}\,\Biggl(\frac{1}{\Delta_{a,j+1}}-\frac{1}{\Delta_{a,j}}\Biggr)\\
    <&\;\frac{2}{\Delta_{a,N_{a}}}.
\end{align*}
By combining these results, we have, almost surely, that
\begin{align*}
    \sum_{b=1}^{B}\Delta_{A_{b}}\I\{\EventF_{b}\}
    &=\sum_{b=1}^{b_{0}-1}\Delta_{A_{b}}\I\{\EventF_{b}\}+\sum_{b=b_{0}}^{B}\Delta_{A_{b}}\I\{\EventF_{b}\}\\
    &\leq\frac{K\Deltamax}{U}+\sum_{i=1}^{\infty}\sum_{b=b_{0}}^{B}\Delta_{A_{b}}\I\{\EventG_{b,i}\}\\
    &\leq\frac{K\Deltamax}{U}+\sum_{i=1}^{\infty}\sum_{a\in\Lambda}\frac{2\gamma\alpha^{i}\ln(BU)}{\beta^{i}\Delta_{a,N_{a}}}\\
    &\leq\frac{K\Deltamax}{U}+2\gamma\ln(BU)\sum_{a\in\Lambda}\frac{1}{\Delta_{a,N_{a}}}\sum_{i=1}^{\infty}\,\Bigl(\frac{\alpha}{\beta}\Bigr)^{i}\\
    &<\frac{K\Deltamax}{U}+2136\ln(BU)\sum_{a\in\Lambda}\frac{1}{\Delta_{a,N_{a}}}\\
    &\leq\frac{K\Deltamax}{U}+\frac{2136(K-U)\ln(BU)}{\Deltamin}
\end{align*}
where the last inequality holds because for all arms $a\in\Lambda$, we have $\smash{\Delta_{a,N_{a}}\!\geq\mu_{U}-\mu_{a}\geq\Deltamin}$. As such,
\begin{equation*}
    \sum_{a\in\Lambda}\frac{1}{\Delta_{a,N_{a}}}\leq\frac{K-U}{\Delta_{\min}}.\qedhere
\end{equation*}
\paragraph{Note.} Observe that when $U=1$, we have $\Delta_{a,N_{a}}=\Delta_{a}$. As such, we can replace the last inequality using
\begin{equation*}
    \ln(BU)\sum_{a\in\Lambda}\frac{1}{\Delta_{a,N_{a}}}
    =\sum_{a:\Delta_{a}>0}\frac{\ln(T)}{\Delta_{a}}
\end{equation*}
to obtain the bound for the classic \UCBi{} algorithm.
\end{proof}

\vspace{2em}

\upperbounddependent*

\begin{proof}
We can bound the regret by
\begin{align*}
    R_{u,T}
    &\leq\sum_{b=1}^{B}\E[\Delta_{A_{b}}\I\{\EventF_{b}\}]+\frac{\pi^{2}K\Deltamax}{3U^{3}}\\
    &\leq\frac{2136(K-U)\ln(BU)}{\Deltamin}+\frac{K\Deltamax}{U}+\frac{\pi^{2}K\Deltamax}{3U^{3}}\\
    &\leq\frac{2136(K-U)\ln(T)}{\Deltamin}+\frac{4K\Deltamax}{U}
\end{align*}
where the first inequality holds due to Lemma~\ref{lem:upperbound-part1} and the second inequality holds due to Lemma~\ref{lem:upperbound-part2}.
\end{proof}

\upperboundindependent*

\begin{proof}
Set
\begin{equation*}
    \delta=\sqrt{\frac{2136(K-U)\ln(BU)}{B}}.
\end{equation*}
Since $\I\{\Delta_{A_{b}}\!<\delta\}+\I\{\Delta_{A_{b}}\!\geq\delta\}=1$, we have
\begin{align*}
    \sum_{b=b_{0}}^{B}\E[\Delta_{A_{b}}\I\{\EventF_{b}\}]
    &=\sum_{b=b_{0}}^{B}\E[\Delta_{A_{b}}\I\{\EventF_{b}\}\I\{\Delta_{A_{b}}\!<\delta\}]\\
    &\hspace{25pt}+\sum_{b=b_{0}}^{B}\E[\Delta_{A_{b}}\I\{\EventF_{b}\}\I\{\Delta_{A_{b}}\!\geq\delta\}].
\end{align*}
We can trivially bound $\E[\Delta_{A_{b}}\I\{\EventF_{b}\}\I\{\Delta_{A_{b}}\!<\delta\}]$ by $\delta$. The second term can be bounded similarly as in Theorem~\ref{lem:upperbound-part2} by using $\Delta_{A_{b}}\geq\delta$ instead of $\Delta_{A_{b}}\geq\Deltamin$. Then, we have
\begin{align*}
    \sum_{b=b_{0}}^{B}\E[\Delta_{A_{b}}\I\{\EventF_{b}\}]
    &\leq B\delta+\frac{2136(K-U)\ln(BU)}{\delta}\\
    &=\sqrt{8544(K-U)\cdot B\ln(BU)}.
\end{align*}
Furthermore, since $\mu_{a}\in[0,1]$ for all $a\in[K]$, we have
\begin{equation*}
    \Deltamax\leq\min\{U,K-U\}.
\end{equation*}
To understand this bound, note that $\Deltamax\leq U$ generally, but when $K<2U$, this can be tightened because there are overlaps between the top $U$ arms and the bottom $U$ arms. This works out to $\Deltamax\leq K-U$. By considering the remaining terms in the regret, we have
\begin{align*}
    R_{T}
    &\leq\sqrt{8544(K-U)\cdot B\ln(BU)}+\frac{4K\Deltamax}{U}\\
    &\leq\sqrt{\frac{8544(K-U)\cdot T\ln(T)}{U}}+\frac{4K\min\{U,K-U\}}{U}.\qedhere
\end{align*}
\end{proof}


\clearpage

\section{Proofs of Lower Bound}
\label{sec:lowerproofs}

\lowerboundi*

\begin{proof}
Note that whenever we play some $A\notin\BoGamma$ under $\nu$, there will be at least $U/2$ users who will incur an instantaneous regret of at least $\Delta$. Under $\EventH$, the total number of sub-optimal arms played across all users and all time steps is at least $TU/4$. By the pigeonhole principle, we know that at least one user played sub-optimal arms for at least $T/4$ times. As such, the regret is at least $\Delta T/4$. A similar argument can be used to show that under $\nu'$ and $\EventH^{\comp}$, the regret is at least $\Delta T/4$. Thus
\begin{align*}
    R_{T,\pi,\nu}+R_{T,\pi,\nu'}
    &>\frac{\Delta T}{4}\bigl(\P_{\pi\nu}(\EventH)+\P_{\pi\nu'}(\EventH^{\comp})\bigr)\\
    &\geq\frac{\Delta T}{8}\exp\bigl(-\KL(\P_{\pi\nu}\|\P_{\pi\nu'})\bigr)
\end{align*}
where the last inequality holds due to the Bretagnolle--Huber inequality.
\end{proof}

\kldecompose*

\begin{proof}
Using the definition of the KL-divergence and applying the chain rule, we have
\begin{align*}
    \KL(\P_{\pi\nu}\|\P_{\pi\nu'})
    &=\E_{\pi\nu}\Biggl[\ln\Biggl(\frac{\dd\P_{\pi\nu}}{\dd\P_{\pi\nu'}}\Biggr)\Biggr]\\
    &=\E_{\pi\nu}\Biggl[\ln\Biggl(\frac{\dd\P_{\pi\nu}/\dd(\rho^{U}\times\lambda^{U})^{T}}{\dd\P_{\pi\nu'}/\dd(\rho^{U}\times\lambda^{U})^{T}}\Biggr)\Biggr].
\end{align*}
\clearpage
We then substitute the Radon-Nikodym derivatives and simplify the terms to get
\begin{align*}
    &\ \E_{\pi\nu}\Biggl[\ln\left(\frac{\dd\P_{\pi\nu}/\dd(\rho^{U}\times\lambda^{U})^{T}}{\dd\P_{\pi\nu'}/\dd(\rho^{U}\times\lambda^{U})^{T}}\right)\Biggr]\\
    =&\ \E_{\pi\nu}\Biggl[\ln\prod_{t=1}^{T}\frac{\pi(A_{t}|A_{1},X_{1},\dots,A_{t-1},X_{t-1})}{\pi(A_{t}|A_{1},X_{1},\dots,A_{t-1},X_{t-1})}\prod_{u=1}^{U}\frac{p_{A_{u,t}}(X_{u,t})}{p'_{A_{u,t}}(X_{u,t})}\Biggr]\\
    =&\ \E_{\pi\nu}\Biggl[\sum_{t=1}^{T}\sum_{u=1}^{U}\ln\frac{p_{A_{u,t}}(X_{u,t})}{p'_{A_{u,t}}(X_{u,t})}\Biggr]\\
    =&\ \sum_{t=1}^{T}\sum_{u=1}^{U}\E_{\pi\nu}\Biggl[\E_{\pi\nu}\Biggl[\ln\frac{p_{A_{u,t}}(X_{u,t})}{p'_{A_{u,t}}(X_{u,t})}\,\bigg\vert\,A_{u,t}\Biggr]\Biggr]\\
    =&\ \sum_{t=1}^{T}\sum_{u=1}^{U}\E_{\pi\nu}\bigl[\KL(\P_{A_{u,t}}\|\P'_{A_{u,t}})\bigr]\\
    =&\ \sum_{t=1}^{T}\E_{\pi\nu}\Biggl[\sum_{a\in A_{t}}\KL(\P_{a}\|\P'_{a})\Biggr]\\
    =&\ \sum_{t=1}^{T}\E_{\pi\nu}\Biggl[\sum_{A\subseteqU\mathcal[K]}\sum_{a\in A}\KL(\P_{a}\|\P'_{a})\cdot\I[A_{t}=A]\Biggr]\\
    =&\ \sum_{A\subseteqU\mathcal[K]}\sum_{a\in A}\KL(\P_{a}\|\P'_{a})\sum_{t=1}^{T}\E_{\pi\nu}\bigl[\I[A_{t}=A]\bigr]\\
    =&\ \sum_{A\subseteqU\mathcal[K]}\sum_{a\in A}\KL(\P_{a}\|\P'_{a})\cdot\E_{\pi\nu}[T_{A,T}]
\end{align*}

\noindent Note that the KL-divergence between two Gaussian measures with mean $\mu_1$ and $\mu_2$ and variance $1$ is $(\mu_1-\mu_2)^2$. Thus, we have
\begin{align*}
    &\ \sum_{A\subseteqU\mathcal[K]}\sum_{a\in A}\KL(\P_{a}\|\P'_{a})\cdot\E_{\pi\nu}[T_{A,T}]\\
    =&\ \sum_{A\subseteqU\mathcal[K]}\E_{\pi\nu}[T_{A,T}]\sum_{a\in A}\KL(\P_{a}\|\P'_{a})\\
    =&\ \sum_{A\subseteqU\mathcal[K]}\E_{\pi\nu}[T_{A,T}]\cdot 4\Delta^{2}|\{a'\in A'|a'\in A\}|\\
    =&\ 4\Delta^{2}\sum_{a'\in A'}\sum_{A:a'\in A}\E_{\pi\nu}[T_{A,T}]\\
    =&\ 4\Delta^{2}\sum_{a'\in A'}\E_{\pi\nu}[T_{a',T}].\qedhere
\end{align*}
\end{proof}

\vspace{2em}

\tatcombi*

\begin{proof}
Suppose, for sake of contradiction, that
\begin{equation*}
    \sum_{a\in A'}T_{a,T}>\frac{TU^{2}}{K-U}.
\end{equation*}
Note that since $A'$ is the set of least played arms, we have
\begin{align*}
    \sum_{A\subseteqU[K]\backslash[U]}\sum_{a\in A}T_{a,T}
    &>\sum_{A\subseteqU[K]\backslash[U]}\frac{TU^{2}}{K-U}\\
    &=\binom{K-U}{U}\frac{TU^{2}}{K-U}.
\end{align*}
Furthermore, the same quantity can be upper bounded by
\begin{align*}
    \sum_{A\subseteqU[K]\backslash[U]}\sum_{a\in A}T_{a,T}
    &=\sum_{a\in[K]\backslash[U]}\sum_{A:a\in A}T_{a,T}\\
    &=\sum_{a\in[K]\backslash[U]}\binom{K-U}{U}\frac{U}{K-U}T_{a,T}\\
    &=\binom{K-U}{U}\frac{U}{K-U}\sum_{a\in[K]\backslash[U]}T_{a,T}\\
    &\leq\binom{K-U}{U}\frac{U}{K-U}\sum_{a\in[K]} T_{a,T}\\
    &=\binom{K-U}{U}\frac{TU^2}{K-U},
\end{align*}
which is a contradiction.
\end{proof}

\lowerbound*

\begin{proof}
We have
\begin{align*}
R_{T,\pi,\nu}+R_{T,\pi,\nu'}
&\geq\frac{\Delta T}{8}\exp\Biggl(-4\Delta^{2}\sum_{a'\in A'}\E_{\pi\nu}[T_{a',T}]\Biggr)\\
&\geq\frac{\Delta T}{8}\exp\Biggl(-\frac{4\Delta^2TU^2}{K-U}\Biggr)\\
&=\frac{\Delta T}{8}\exp(-1/2)\\
&>\frac{\sqrt{T(K-U)}}{38U}.
\end{align*}
Since $2\max\{R_{T,\pi,\nu},R_{T,\pi,\nu'}\}\geq R_{T,\pi,\nu}+R_{T,\pi,\nu'}$, dividing by $2$ concludes the proof.
\end{proof}

%% file: main.bbl
\begin{thebibliography}{37}
\providecommand{\natexlab}[1]{#1}
\providecommand{\url}[1]{\texttt{#1}}
\expandafter\ifx\csname urlstyle\endcsname\relax
  \providecommand{\doi}[1]{doi: #1}\else
  \providecommand{\doi}{doi: \begingroup \urlstyle{rm}\Url}\fi

\bibitem[\'{A}lvaro Gonz\'{a}lez et~al.(2022)\'{A}lvaro Gonz\'{a}lez, Ortega, P\'{e}rez-L\'{o}pez, and Alonso]{gonzalez2022}
\'{A}lvaro Gonz\'{a}lez, Fernando Ortega, Diego P\'{e}rez-L\'{o}pez, and Santiago Alonso.
\newblock Bias and unfairness of collaborative filtering based recommender systems in movielens dataset.
\newblock \emph{IEEE Access}, 10:\penalty0 68429--68439, 2022.
\newblock \doi{10.1109/ACCESS.2022.3186719}.

\bibitem[Anantharam et~al.(1987)Anantharam, Varaiya, and Walrand]{anantharam1987}
V.~Anantharam, P.~Varaiya, and J.~Walrand.
\newblock Asymptotically efficient allocation rules for the multiarmed bandit problem with multiple plays-part i: I.i.d. rewards.
\newblock \emph{IEEE Transactions on Automatic Control}, 32\penalty0 (11):\penalty0 968--976, 1987.
\newblock \doi{10.1109/TAC.1987.1104491}.

\bibitem[Audibert \& Bubeck(2009)Audibert and Bubeck]{audibert2009}
Jean-Yves Audibert and S\'{e}bastien Bubeck.
\newblock Minimax policies for adversarial and stochastic bandits.
\newblock In \emph{Proceedings of the 22nd Annual Conference on Learning Theory}, Berlin, Germany, 2009. Springer-Verlag.

\bibitem[Auer et~al.(2002)Auer, Cesa-Bianchi, and Fischer]{auer2002}
Peter Auer, Nicol\`{o} Cesa-Bianchi, and Paul Fischer.
\newblock Finite-time analysis of the multiarmed bandit problem.
\newblock \emph{Mach. Learn.}, 47\penalty0 (2–3):\penalty0 235–256, May 2002.
\newblock \doi{10.1023/A:1013689704352}.

\bibitem[Avner \& Mannor(2014)Avner and Mannor]{avner2014}
Orly Avner and Shie Mannor.
\newblock Concurrent bandits and cognitive radio networks, 2014.

\bibitem[Cesa-Bianchi \& Lugosi(2012)Cesa-Bianchi and Lugosi]{cesa2012}
Nicol\`{o} Cesa-Bianchi and G\'{a}bor Lugosi.
\newblock Combinatorial bandits.
\newblock \emph{Journal of Computer and System Sciences}, 78\penalty0 (5):\penalty0 1404--1422, 2012.
\newblock \doi{https://doi.org/10.1016/j.jcss.2012.01.001}.

\bibitem[Chen et~al.(2016)Chen, Wang, Yuan, and Wang]{chen2016}
Wei Chen, Yajun Wang, Yang Yuan, and Qinshi Wang.
\newblock Combinatorial multi-armed bandit and its extension to probabilistically triggered arms.
\newblock \emph{J. Mach. Learn. Res.}, 17\penalty0 (1):\penalty0 1746–1778, Jan 2016.

\bibitem[Chen et~al.(2020)Chen, Cuellar, Luo, Modi, Nemlekar, and Nikolaidis]{chen2020}
Yifang Chen, Alex Cuellar, Haipeng Luo, Jignesh Modi, Heramb Nemlekar, and Stefanos Nikolaidis.
\newblock The fair contextual multi-armed bandit.
\newblock In \emph{Proceedings of the 19th International Conference on Autonomous Agents and MultiAgent Systems}, AAMAS '20, pp.\  1810–1812, California, United States, 2020. International Foundation for Autonomous Agents and Multiagent Systems.

\bibitem[Claure et~al.(2020)Claure, Chen, Modi, Jung, and Nikolaidis]{claure2020}
Houston Claure, Yifang Chen, Jignesh Modi, Malte Jung, and Stefanos Nikolaidis.
\newblock Multi-armed bandits with fairness constraints for distributing resources to human teammates.
\newblock In \emph{Proceedings of the 2020 ACM/IEEE International Conference on Human-Robot Interaction}, HRI '20, pp.\  299–308, Cambridge, United Kingdom, 2020. Association for Computing Machinery.
\newblock \doi{10.1145/3319502.3374806}.
\newblock URL \url{https://doi.org/10.1145/3319502.3374806}.

\bibitem[Ding et~al.(2019)Ding, Li, and Liu]{ding2019}
Kaize Ding, Jundong Li, and Huan Liu.
\newblock Interactive anomaly detection on attributed networks.
\newblock In \emph{Proceedings of the Twelfth ACM International Conference on Web Search and Data Mining}, WSDM '19, pp.\  357–365, New York, United States, 2019. Association for Computing Machinery.
\newblock \doi{10.1145/3289600.3290964}.

\bibitem[Durand et~al.(2018)Durand, Achilleos, Iacovides, Strati, Mitsis, and Pineau]{durand2018}
Audrey Durand, Charis Achilleos, Demetris Iacovides, Katerina Strati, Georgios~D. Mitsis, and Joelle Pineau.
\newblock Contextual bandits for adapting treatment in a mouse model of de novo carcinogenesis.
\newblock In Finale Doshi-Velez, Jim Fackler, Ken Jung, David Kale, Rajesh Ranganath, Byron Wallace, and Jenna Wiens (eds.), \emph{Proceedings of the 3rd Machine Learning for Healthcare Conference}, volume~85 of \emph{Proceedings of Machine Learning Research}, pp.\  67--82, California, United States, 2018. PMLR.
\newblock URL \url{https://proceedings.mlr.press/v85/durand18a.html}.

\bibitem[Forouzandeh et~al.(2021)Forouzandeh, Berahmand, and Rostami]{forouzandeh2021}
Saman Forouzandeh, Kamal Berahmand, and Mehrdad Rostami.
\newblock Presentation of a recommender system with ensemble learning and graph embedding: A case on movielens.
\newblock \emph{Multimedia Tools and Applications}, 80\penalty0 (5):\penalty0 7805--7832, Feb 2021.
\newblock \doi{10.1007/s11042-020-09949-5}.

\bibitem[Gai et~al.(2012)Gai, Krishnamachari, and Jain]{gai2012}
Yi~Gai, Bhaskar Krishnamachari, and Rahul Jain.
\newblock Combinatorial network optimization with unknown variables: Multi-armed bandits with linear rewards and individual observations.
\newblock \emph{IEEE/ACM Transactions on Networking}, 20\penalty0 (5):\penalty0 1466--1478, 2012.
\newblock \doi{10.1109/TNET.2011.2181864}.

\bibitem[He et~al.(2017)He, Liao, Zhang, Nie, Hu, and Chua]{he2017}
Xiangnan He, Lizi Liao, Hanwang Zhang, Liqiang Nie, Xia Hu, and Tat-Seng Chua.
\newblock Neural collaborative filtering.
\newblock In \emph{Proceedings of the 26th International Conference on World Wide Web}, WWW '17, pp.\  173–182, Geneva, Switzerland, 2017. International World Wide Web Conferences Steering Committee.
\newblock \doi{10.1145/3038912.3052569}.

\bibitem[Helmbold \& Warmuth(2009)Helmbold and Warmuth]{helmbold2009}
David~P. Helmbold and Manfred~K. Warmuth.
\newblock Learning permutations with exponential weights.
\newblock \emph{Journal of Machine Learning Research}, 10\penalty0 (58):\penalty0 1705--1736, 2009.
\newblock URL \url{http://jmlr.org/papers/v10/helmbold09a.html}.

\bibitem[Hossain et~al.(2021)Hossain, Micha, and Shah]{hossain2021}
Safwan Hossain, Evi Micha, and Nisarg Shah.
\newblock Fair algorithms for multi-agent multi-armed bandits.
\newblock In M.~Ranzato, A.~Beygelzimer, Y.~Dauphin, P.S. Liang, and J.~Wortman Vaughan (eds.), \emph{Advances in Neural Information Processing Systems}, volume~34, pp.\  24005--24017, New York, United States, 2021. Curran Associates, Inc.
\newblock URL \url{https://proceedings.neurips.cc/paper_files/paper/2021/file/c96ebeee051996333b6d70b2da6191b0-Paper.pdf}.

\bibitem[Joseph et~al.(2016)Joseph, Kearns, Morgenstern, and Roth]{joseph2016}
Matthew Joseph, Michael Kearns, Jamie Morgenstern, and Aaron Roth.
\newblock Fairness in learning: Classic and contextual bandits.
\newblock In D.~Lee, M.~Sugiyama, U.~Luxburg, I.~Guyon, and R.~Garnett (eds.), \emph{Advances in Neural Information Processing Systems}, volume~29, New York, United States, 2016. Curran Associates, Inc.
\newblock URL \url{https://proceedings.neurips.cc/paper_files/paper/2016/file/eb163727917cbba1eea208541a643e74-Paper.pdf}.

\bibitem[Joseph et~al.(2018)Joseph, Kearns, Morgenstern, Neel, and Roth]{joseph2018}
Matthew Joseph, Michael Kearns, Jamie Morgenstern, Seth Neel, and Aaron Roth.
\newblock Meritocratic fairness for infinite and contextual bandits.
\newblock In \emph{Proceedings of the 2018 AAAI/ACM Conference on AI, Ethics, and Society}, AIES '18, pp.\  158–163, New York, United States, 2018. Association for Computing Machinery.
\newblock \doi{10.1145/3278721.3278764}.

\bibitem[Jouini et~al.(2009)Jouini, Ernst, Moy, and Palicot]{wassim2009}
Wassim Jouini, Damien Ernst, Christophe Moy, and Jacques Palicot.
\newblock Multi-armed bandit based policies for cognitive radio's decision making issues.
\newblock In \emph{2009 3rd International Conference on Signals, Circuits and Systems (SCS)}, Medenine, Tunisia, 2009. Institute of Electrical and Electronics Engineers (IEEE).
\newblock \doi{10.1109/ICSCS.2009.5412697}.

\bibitem[Komiyama et~al.(2015)Komiyama, Honda, and Nakagawa]{komiyama2015}
Junpei Komiyama, Junya Honda, and Hiroshi Nakagawa.
\newblock Optimal regret analysis of thompson sampling in stochastic multi-armed bandit problem with multiple plays.
\newblock In \emph{Proceedings of the 32nd International Conference on International Conference on Machine Learning}, volume~37 of \emph{ICML'15}, pp.\  1152–1161, Lille, France, 2015. JMLR.org.

\bibitem[Koolen et~al.(2010)Koolen, Warmuth, and Kivinen]{koolen2010}
Wouter~M. Koolen, Manfred~K. Warmuth, and Jyrki Kivinen.
\newblock Hedging structured concepts.
\newblock \emph{Virtual Reality}, 2010.
\newblock URL \url{https://api.semanticscholar.org/CorpusID:70290123}.

\bibitem[Kousiouris et~al.(2011)Kousiouris, Cucinotta, and Varvarigou]{kousiouris2011}
George Kousiouris, Tommaso Cucinotta, and Theodora Varvarigou.
\newblock The effects of scheduling, workload type and consolidation scenarios on virtual machine performance and their prediction through optimized artificial neural networks.
\newblock \emph{Journal of Systems and Software}, 84\penalty0 (8):\penalty0 1270--1291, 2011.
\newblock \doi{https://doi.org/10.1016/j.jss.2011.04.013}.

\bibitem[Ku\.{z}elewska(2014)]{kuzelewska2014}
Urszula Ku\.{z}elewska.
\newblock Clustering algorithms in hybrid recommender system on movielens data.
\newblock \emph{Studies in Logic, Grammar and Rhetoric}, 37:\penalty0 125--139, Jan 2014.
\newblock \doi{10.2478/slgr-2014-0021}.

\bibitem[Kveton et~al.(2014)Kveton, Wen, Ashkan, Eydgahi, and Eriksson]{kveton2014}
Branislav Kveton, Zheng Wen, Azin Ashkan, Hoda Eydgahi, and Brian Eriksson.
\newblock Matroid bandits: fast combinatorial optimization with learning.
\newblock In \emph{Proceedings of the Thirtieth Conference on Uncertainty in Artificial Intelligence}, UAI'14, pp.\  420–429, 2014.

\bibitem[Kveton et~al.(2015{\natexlab{a}})Kveton, Wen, Ashkan, and Szepesvari]{kveton2015a}
Branislav Kveton, Zheng Wen, Azin Ashkan, and Csaba Szepesvari.
\newblock Tight regret bounds for stochastic combinatorial semi-bandits.
\newblock In Guy Lebanon and S.~V.~N. Vishwanathan (eds.), \emph{Proceedings of the Eighteenth International Conference on Artificial Intelligence and Statistics}, volume~38 of \emph{Proceedings of Machine Learning Research}, pp.\  535--543, California, United States, 2015{\natexlab{a}}. PMLR.
\newblock URL \url{https://proceedings.mlr.press/v38/kveton15.html}.

\bibitem[Kveton et~al.(2015{\natexlab{b}})Kveton, Wen, Ashkan, and Szepesv\'{a}ri]{kveton2015b}
Branislav Kveton, Zheng Wen, Azin Ashkan, and Csaba Szepesv\'{a}ri.
\newblock Combinatorial cascading bandits.
\newblock In \emph{Proceedings of the 28th International Conference on Neural Information Processing Systems}, volume~1 of \emph{NIPS'15}, pp.\  1450–1458, Massachusetts, United States, 2015{\natexlab{b}}. MIT Press.

\bibitem[Li et~al.(2020)Li, Liu, and Ji]{li2020}
Fengjiao Li, Jia Liu, and Bo~Ji.
\newblock Combinatorial sleeping bandits with fairness constraints.
\newblock \emph{IEEE Transactions on Network Science and Engineering}, 7\penalty0 (3):\penalty0 1799--1813, 2020.
\newblock \doi{10.1109/TNSE.2019.2954310}.

\bibitem[Liu \& Zhao(2010)Liu and Zhao]{liu2010}
Keqin Liu and Qing Zhao.
\newblock Distributed learning in multi-armed bandit with multiple players.
\newblock \emph{IEEE Transactions on Signal Processing}, 58\penalty0 (11):\penalty0 5667--5681, 2010.
\newblock \doi{10.1109/TSP.2010.2062509}.

\bibitem[Mueller et~al.(2019)Mueller, Syrgkanis, and Taddy]{mueller2019}
Jonas Mueller, Vasilis Syrgkanis, and Matt Taddy.
\newblock Low-rank bandit methods for high-dimensional dynamic pricing.
\newblock In \emph{Proceedings of the 33rd International Conference on Neural Information Processing Systems}, New York, United States, 2019. Curran Associates Inc.

\bibitem[Rawls(1971)]{rawls1971}
John Rawls.
\newblock \emph{A Theory of Justice: Original Edition}.
\newblock Harvard University Press, Massachusetts, United States, 1971.
\newblock URL \url{http://www.jstor.org/stable/j.ctvjf9z6v}.

\bibitem[Sawarni et~al.(2023)Sawarni, Pal, and Barman]{sawarni2023}
Ayush Sawarni, Soumybrata Pal, and Siddharth Barman.
\newblock Nash regret guarantees for linear bandits, 2023.

\bibitem[Shen et~al.(2015)Shen, Wang, Jiang, and Zha]{shen2015}
Weiwei Shen, Jun Wang, Yu-Gang Jiang, and Hongyuan Zha.
\newblock Portfolio choices with orthogonal bandit learning.
\newblock In \emph{Proceedings of the 24th International Conference on Artificial Intelligence}, IJCAI'15, pp.\  974–980, California, United States, 2015. AAAI Press.

\bibitem[Wang et~al.(2021)Wang, Bai, Sun, and Joachims]{wang2021}
Lequn Wang, Yiwei Bai, Wen Sun, and Thorsten Joachims.
\newblock Fairness of exposure in stochastic bandits, 2021.

\bibitem[Wang \& Chen(2018)Wang and Chen]{wang18}
Siwei Wang and Wei Chen.
\newblock Thompson sampling for combinatorial semi-bandits.
\newblock In Jennifer Dy and Andreas Krause (eds.), \emph{Proceedings of the 35th International Conference on Machine Learning}, volume~80 of \emph{Proceedings of Machine Learning Research}, pp.\  5114--5122. PMLR, Jul 2018.
\newblock URL \url{https://proceedings.mlr.press/v80/wang18a.html}.

\bibitem[Wang et~al.(2022)Wang, Xie, and Lui]{wang2022}
Xuchuang Wang, Hong Xie, and John C.~S. Lui.
\newblock Multiple-play stochastic bandits with shareable finite-capacity arms.
\newblock In Kamalika Chaudhuri, Stefanie Jegelka, Le~Song, Csaba Szepesvari, Gang Niu, and Sivan Sabato (eds.), \emph{Proceedings of the 39th International Conference on Machine Learning}, volume 162 of \emph{Proceedings of Machine Learning Research}, pp.\  23181--23212, Maryland, United States, 2022. PMLR.
\newblock URL \url{https://proceedings.mlr.press/v162/wang22af.html}.

\bibitem[Wen et~al.(2017)Wen, Kveton, Valko, and Vaswani]{zheng2017}
Zheng Wen, Branislav Kveton, Michal Valko, and Sharan Vaswani.
\newblock Online influence maximization under independent cascade model with semi-bandit feedback.
\newblock In \emph{Proceedings of the 31st International Conference on Neural Information Processing Systems}, NIPS'17, pp.\  3026–3036, New York, USA, 2017. Curran Associates Inc.

\bibitem[Wilkes(2020)]{wilkes2020}
John Wilkes.
\newblock Yet more {Google} compute cluster trace data.
\newblock Google research blog, Apr 2020.
\newblock Posted at \url{https://ai.googleblog.com/2020/04/yet-more-google-compute-cluster-trace.html}.

\end{thebibliography}
